\ifdefined\pdfminorversion
\fi
\PassOptionsToPackage{table}{xcolor}
\PassOptionsToPackage{capitalize}{cleveref}
\documentclass{fairmeta}

\usepackage{hyperref}
\usepackage{url}            
\usepackage{booktabs}       
\usepackage{amsfonts}       
\usepackage{nicefrac}       
\usepackage[table]{xcolor}
\usepackage{graphicx}
\usepackage{multirow}
\usepackage{tabularx}
\usepackage{array} 
\ifPDFTeX
  \usepackage[accsupp]{axessibility}
\fi 
\usepackage{enumitem}

\definecolor{colorA}{RGB}{235,245,255}
\definecolor{colorB}{RGB}{237, 255, 208}
\newcommand{\eg}{\textit{e.g.}\@\xspace}
\newcommand{\ie}{\textit{i.e.}\@\xspace}

\usepackage{amsmath,amssymb,amsfonts,amsthm}
\usepackage{bm}

\newtheorem{theorem}{Theorem}
\newtheorem{lemma}[theorem]{Lemma}
\newtheorem{proposition}[theorem]{Proposition}

\newtheorem{definition}[theorem]{Definition}

\usepackage{xcolor}
\usepackage{tikz}
\usetikzlibrary{arrows.meta, positioning, calc, fit, backgrounds}
\usepackage{comment}
\usepackage{xspace}

\makeatletter
\DeclareRobustCommand\onedot{\futurelet\@let@token\@onedot}
\def\@onedot{\ifx\@let@token.\else.\null\fi\xspace}
\def\eg{\emph{e.g}\onedot}

\def\ie{\emph{i.e}\onedot}

\makeatother

\newcommand{\R}{\mathbb{R}}

\newcommand{\cZ}{\mathcal{Z}}

\newcommand{\cL}{\mathcal{L}}

\newcommand{\cF}{\mathcal{F}}

\DeclareMathOperator{\E}{\mathbb{E}}

\DeclareMathOperator*{\argmin}{arg\,min}

\newcommand{\vfield}{v_{\theta}}

\newcommand{\Lfm}{\cL_{\mathrm{FM}}}

\newcommand{\ours}{DRIFT\xspace}        

\newcommand{\Rad}{\mathfrak{R}}
\newcommand{\cH}{\mathcal{H}}
\newcommand{\Risk}{\mathcal{R}}
\newcommand{\App}{\mathrm{App}}

\usepackage[capitalize]{cleveref}

\title{DRIFT: A Residual Flow Adapter for Decoding Continuous Outputs in Vision-Language Models}

\author[1*]{\href{https://dragonlzm.github.io/zhuomingliu.github.io/}{Zhuoming Liu}}
\author[1*]{\href{https://jonneslin.github.io/}{Jinhong Lin}}
\author[1*]{\href{https://scholar.google.com/citations?user=jLGzlMEAAAAJ&hl=en}{Kwan Man Cheng}}
\author[1]{\href{https://lzhangbj.github.io/}{Lin Zhang}}
\author[2]{\href{https://scholar.google.com/citations?user=KRm87PsAAAAJ}{Shayok Bagchi}}
\author[1]{\href{https://scholar.google.com/citations?user=_y-8nrcAAAAJ}{Yin Li}}
\affiliation[1]{University of Wisconsin-Madison}
\affiliation[2]{West Lafayette Jr./Sr. High School}
\contribution[*]{Co-first Author}

\metadata[Project page]{\url{https://dragonlzm.github.io/DRIFT}}
\metadata[Code]{\url{https://github.com/dragonlzm/DRIFT}}

\abstract{%
Many modern vision-language models (VLMs) build on autoregressive decoding of discrete tokens. While text-based output interfaces enable scalable pretraining and strong zero-shot generalization across diverse tasks, they are poorly suited for problems that require precise continuous outputs, such as localizing temporal boundaries of events or generating robotic control actions. To address this challenge, we propose DRIFT, a general framework for adapting pretrained VLMs to continuous decoding tasks.  DRIFT combines a base predictor, which provides a coarse estimate of the target output, with a generative refinement module based on flow matching that iteratively improves the prediction. This residual formulation transforms the generative modeling problem from learning a global output distribution to modeling a localized residual distribution around a strong prior, substantially simplifying optimization. We evaluate DRIFT on both perception and planning tasks, including visual grounding and robotic control. Across multiple tasks and architectures spanning MLLMs, VLAs, and WAMs, DRIFT consistently outperforms a strong set of regression- and generative-based solutions. 
}

\begin{document}
\maketitle

\section{Introduction}
\label{sec:intro}

Vision-language models (VLMs)~\citep{liu2023visual, NEURIPS2022_960a172b} pretrained on large-scale visual and textual corpora have achieved remarkable success across perception and planning tasks, leading to the emergence of multimodal large language models (MLLMs)~\citep{bai2025qwen3, deitke2025molmo}, vision-language-action models (VLAs)~\citep{kim2024openvla,black2025pi0,intelligence2025pi}, and world action models (WAMs)~\citep{yuan2026fast, ye2026world}. Modern VLMs typically adopt a generic autoregressive backbone (\eg, Transformers~\citep{vaswani2017attention}) and decode discrete text tokens as their primary output interface. This design enables scalable pretraining via next token prediction, and supports flexible zero-shot transfer to downstream tasks whose outputs can be naturally expressed as text, including classification, captioning, and question answering.

However, many important perception and planning tasks fundamentally require predicting continuous quantities. Examples include estimating event timestamps in temporal video grounding (TVG)~\citep{sigurdsson2016hollywood, krishna2017dense} and generating control signals for robotic manipulation in VLA tasks~\citep{liu2023libero,li24simpler}. The discrete token interface is poorly suited to such tasks: tokenization introduces quantization errors, discards ordinal structure, and ignores fine-grained variation within bins. In this work, we ask the following question: ``\textit{How can pretrained VLMs with discrete decoding interfaces be adapted for precise continuous prediction while preserving and leveraging the knowledge acquired during pretraining?}'' 

\begin{figure}[t!]
    \centering
    \includegraphics[width=0.9\linewidth]{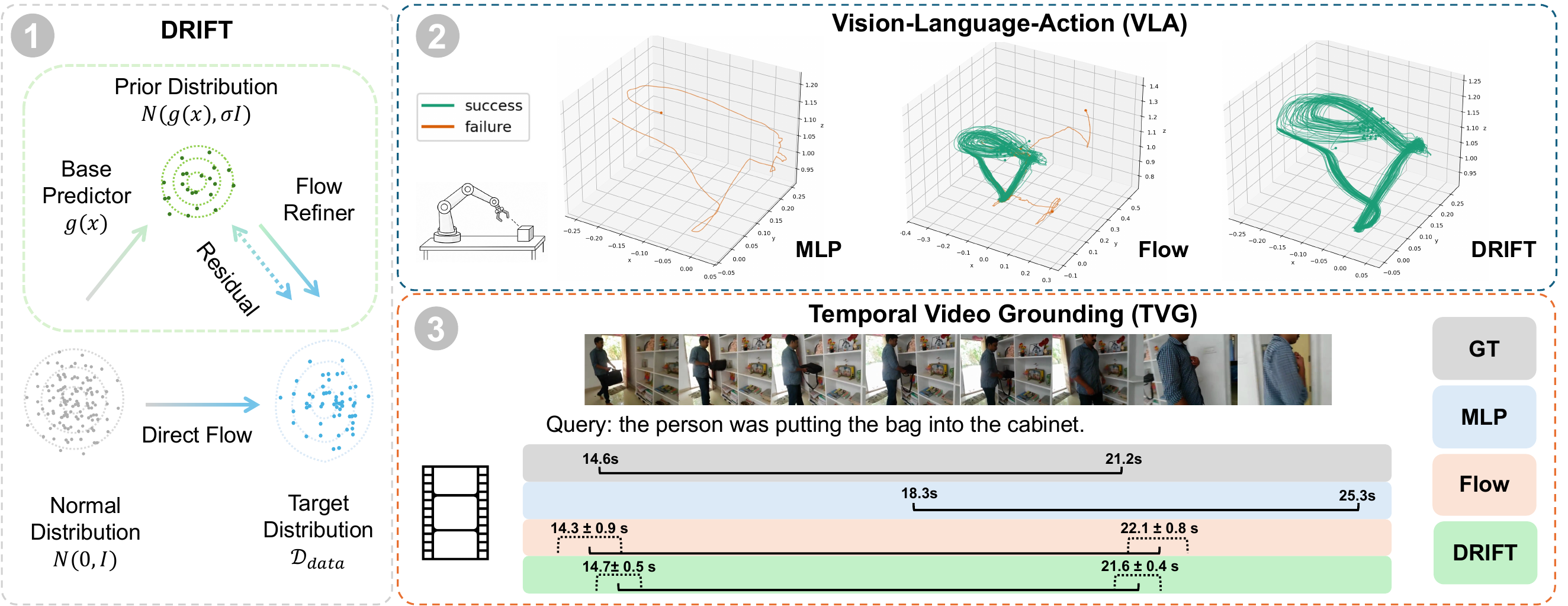}\vspace{-0.5em}
    \caption{\textbf{(1)}: \textbf{\ours} learns a velocity field that transports a starting distribution centered at an initial coarse prediction toward the target distribution, simplifying optimization by modeling localized residuals.
    \textbf{(2, 3)}: \textbf{Visualization} of predicted action trajectories (on Libero-Long) and temporal event boundaries (on Charades-STA) under different decoding strategies. Compared to MLP or flow matching, \ours produces more accurate and stable predictions with lower variance.
    }
    \label{fig:teaser}\vspace{-1.2em}
    
\end{figure}

Existing approaches only partially address this challenge. In MLLMs, prior work often discretizes continuous targets into tokens~\citep{kim2024openvla, bai2025qwen3} or appends special tokens with lightweight regression heads~\citep{zhang2026vlm4vla}. In VLAs and WAMs, diffusion- and flow-based action decoders have emerged as dominant approaches for generating continuous outputs~\citep{black2025pi0,intelligence2025pi,yuan2026fast}. Tokenization and direct regression are deterministic and thus struggle to model uncertainty in continuous prediction, including multimodal output distributions~\citep{chi2025diffusion} and high-variance modes~\citep{moltisanti2017trespassing, zhang2025timelens} (Figure~\ref{fig:teaser}-2,3). Diffusion- and flow-based approaches alleviate this issue by iteratively refining predictions through generative modeling in continuous space. However, they typically require more training data and optimization effort, can exhibit high prediction variance in low-data regimes (Figure~\ref{fig:teaser}-2,3), and must devote modeling capacity to global structure that can be readily captured by deterministic predictors.

\textit{Our key insight} is that deterministic prediction and iterative generative refinement can form a natural cascade. Deterministic predictors leverage pretrained representations to produce efficient coarse estimates, while generative refinement effectively captures uncertainty and recovers fine-grained structure around these estimates. 
Leveraging this insight, we propose \textbf{\ours} (\underline{D}ecoding via \underline{R}es\underline {i}dual \underline{F}low Adap\underline{t}er), a general framework for adapting pretrained VLMs to continuous decoding tasks. \ours introduces a lightweight residual flow adapter that augments a base predictor such as a token decoder or multilayer perceptron (MLP) regressor, with a flow-matching-based refinement module. 
Rather than generating predictions from scratch, \ours models the residual error distribution around an initial coarse prediction and progressively refines it using flow matching (Figure~\ref{fig:teaser}-1).

We evaluate \ours on perception (spatial and temporal grounding with MLLMs) and planning (robotic control with VLAs and WAMs). Across all settings, \ours consistently improves over both direct regression baselines and standalone flow-matching decoders, while achieving state-of-the-art performance on several TVG and VLA benchmarks. On TVG benchmarks, \ours achieves gains of +2.6\% on Charades-STA and +2.1\% on ActivityNet Captions compared to the latest methods~\citep{pramanick2025edvtg}. In VLA tasks, \ours built on Qwen3-VL-2B~\citep{bai2025qwen3} achieves 97.9\% average accuracy on Libero benchmarks and 61.5\% accuracy on Simpler WidowX.

\textbf{Our main contributions} are summarized as follows.
\begin{itemize}[leftmargin=2em, itemsep=2pt, topsep=2pt]
    \item We introduce \textbf{\ours}, a general framework for adapting pretrained vision-language models with discrete decoding interfaces to precise continuous prediction tasks.
    \item We develop a residual refinement adapter that combines efficient coarse prediction with flow-matching-based refinement, together with theoretical analysis of its optimization advantages and an effective training strategy.
    \item Through extensive experiments across visual grounding and robotic manipulation benchmarks, we demonstrate consistent gains over strong baselines across multiple VLM architectures.
\end{itemize}

\section{Related Work}
\label{sec:related}

\textbf{Vision Language Models (VLMs).}
VLMs learn from paired visual and textual data to address a broad range of vision-language tasks. In particular, multi-modal large language models (MLLMs)~\citep{liu2023visual, bai2025qwen3, deitke2025molmo, pmlr-v202-li23q, NEURIPS2022_960a172b} have emerged as powerful frameworks for learning transferable cross-modal representations from massive-scale image-text and video-text datasets, enabling sophisticated multimodal understanding and reasoning~\citep{fu2025video,yue2024mmmu,li2024mvbench}. Beyond passive perception, another line of VLM research focuses on action control, leading to vision-language-action (VLA) models\mbox{~\citep{kim2024openvla, black2025pi0, intelligence2025pi}}. These models are typically pretrained on large-scale action trajectories to improve cross-embodiment generalization and can be adapted to diverse downstream manipulation tasks with limited fine-tuning. More recently, researchers have begun to connect scene generation with action control in world-action models (WAMs)~\citep{yuan2026fast, ye2026world}, which aim to model both environmental dynamics and action-conditioned interactions.

\medskip
\textbf{Vision Tasks with Continuous Output.} 
Many perception and planning tasks require outputs that are inherently continuous rather than purely discrete. In this work, we focus on two representative tasks: (1) visual grounding~\citep{xiao2025towards}, which localizes the spatial or temporal extent of a visual concept in an image or video given a language query, and (2) VLA tasks~\citep{liu2023libero,li24simpler}, which predict continuous robot action trajectories conditioned on visual observations and language instructions.
Prior work has explored several ways to adapt VLMs to such continuous-output tasks. For visual grounding, some methods~\citep{bai2025qwen3, clark2026molmo2, unitime, Sun_2025_ICCV} discretize box coordinates or timestamps into tokens, while others~\citep{pramanick2025edvtg,mu2024snag,lin2023univtg} attach an MLP head to directly regress continuous values. For VLA tasks, existing approaches either use a simple MLP decoder to predict action trajectories from VLM representations~\citep{zhang2026vlm4vla}, or adopt diffusion- or flow-matching-based heads~\citep{black2025pi0, chen2025internvla, zitkovich2023rt} to improve action precision.

Unlike these task-specific decoding strategies, \ours provides an efficient and general adapter for extending pretrained VLMs to continuous-output tasks. By refining continuous predictions around a learned prior, our framework supports diverse perception and planning problems while preserving the reusable Transformer backbone of existing VLMs.

\medskip
\noindent\textbf{Diffusion and Flow Matching.} Diffusion models~\citep{ho2020denoising, song2020score} seek to reverse a forward ``diffusion'' process that progressively injects noise into data by learning a denoising function. 
Flow matching models~\citep{lipman2023flow,liu2022flow} simplify the training of continuous-time normalizing flows as supervised regression over target vector fields, leading to an objective closely connected to diffusion models.
While both frameworks were initially designed for generative modeling~\citep{rombach2022high, polyak2024movie}, recent research has expanded their utility to continuous value decoding in VLA~\citep{chi2025diffusion,black2025pi0,yang2025vlaser}. 

\section{\ours: Decoding Continuous Output via Residual Flow Adapter}
\label{sec:method}

\subsection{Preliminaries: Decoding Continuous Outputs in Autoregressive VLMs}\label{sec:method:background}
\noindent\textbf{Autoregressive VLMs}. An autoregressive VLM takes an image or video $\mathbf{X}^v$ and a text query $\mathbf{X}^q=\{x^q\}$ as its input, and generates an answer $\mathbf{X}^a=\{x^a\}$ in text form. 
Specifically, $\mathbf{X}^v$ is first encoded by a visual encoder $E_v(\cdot)$ (including a vision backbone and its projector) into a set of vision tokens $\{\mathbf{z}^v \in \mathbb{R}^d\}$. 
Similarly, $\mathbf{X}^q$ is processed by a text encoder $E_t(\cdot)$, which embeds the words $\{x^q\}$ into a set of text tokens $\{\mathbf{z}^q \in \mathbb{R}^d\}$.
These tokens are combined into $\{\mathbf{z}^{v|q}\} = [\{\mathbf{z}^v\}, \{\mathbf{z}^q\}]$, and processed by $f(\cdot)$ (often an LLM) that decodes the answer $\mathbf{X}^a$ autoregressively
\begin{equation}
    z_i^a = f\left( \left[ \{\mathbf{z}^{v|q}\}, \{\mathbf{z}^a_{<i}\} \right]; \mathbf{\theta}_f \right), \quad x_i^a = D_t(z_i^a), \label{eq:llm}
\end{equation}
where $D_t(\cdot)$ is the decoding function that converts the embedding $z_i^a$ to the answer token $x_i^a$ in text, and $\{\mathbf{z}^a_{<i}\}$ are text tokens from previously generated answer $x^a_{<i}$, \ie $\mathbf{z}^a = E_t(x^a)$.

\medskip
\noindent\textbf{Decoding Continuous Outputs}. For many vision tasks, $\mathbf{X}^a$ involves a continuous output. For simplicity, we consider a single continuous output $y_i=x_i^a \in \mathbb{R}^c$ at slot $i$ with dimension $c$, though the formulation can be easily extended to multiple variables. Without loss of clarity, we drop the subscript $i$ and the superscript $a$ hereafter. 
To decode a continuous $y$, a separate decoding function $g(z) \rightarrow y$ is often needed.  
When \textit{tokenization} is applied to $y$, $g$ can re-use the discrete decoder $D_t$.   
In \textit{direct regression}, a special token $x^{s}$ representing the continuous output is introduced, and $g$ is realized as a simple MLP. The decoder $g$ is activated only when $x^{s}$ is generated.
When using \textit{diffusion or flow match} heads, a noise or time dependent version of $g$ will be learned to model the conditional probability $p(y|z)$, again, with the help of special tokens $x^{s}$. 

This continuous decoding function $g$ can be provided as part of the modern VLMs, often in the form of tokenization or a simple MLP. For example, Qwen3-VL~\citep{bai2025qwen3} has incorporated visual grounding into pre-training by representing numerical quantities as discrete text tokens.  When $g$ is not provided, we implement $g$ with an MLP that is jointly learned with \ours.

\medskip
\noindent\textbf{Flow Matching}. Flow matching~\citep{lipman2023flow} or diffusion~\citep{chi2025diffusion} heads can be used for continuous decoding. Conditional flow matching learns a velocity field $\vfield(y(t), t, z)\triangleq dy(t)/dt$, typically a neural network with parameters $\theta$, that transports samples from a known noise distribution to a target distribution $p(y|z)$. 

With target $y(t=1) \in \R^c$, noise $y(t=0) \in \R^c \sim \mathcal{N}(\bm 0, \bm I)$, and time $t \sim \mathcal{U}(0,1)$, the linear interpolant $y(t) = (1{-}t)\, y(t=0) + t\,y(t=1)$ has ground-truth velocity $v^* = y(t=1) - y(t=0)$.

The standard flow-matching loss is
\begin{equation}
  \Lfm = \E\Big[\|\vfield(y(t), t, z) - v^* \|_2^2\Big].
  \label{eq:fm:loss}
\end{equation}
Training minimizes \cref{eq:fm:loss} with respect to the parameters $\theta$ in the velocity field $v_\theta(\cdot)$. Inference predicts $y(t=1)$ by solving the ODE $dy/dt = \vfield(y(t), t, z)$ given $y(t=0)$ and $z$.

\subsection{Design of \ours}\label{sec:method:overview}

\begin{figure}[t!]
    \centering
    \includegraphics[width=0.98\linewidth]{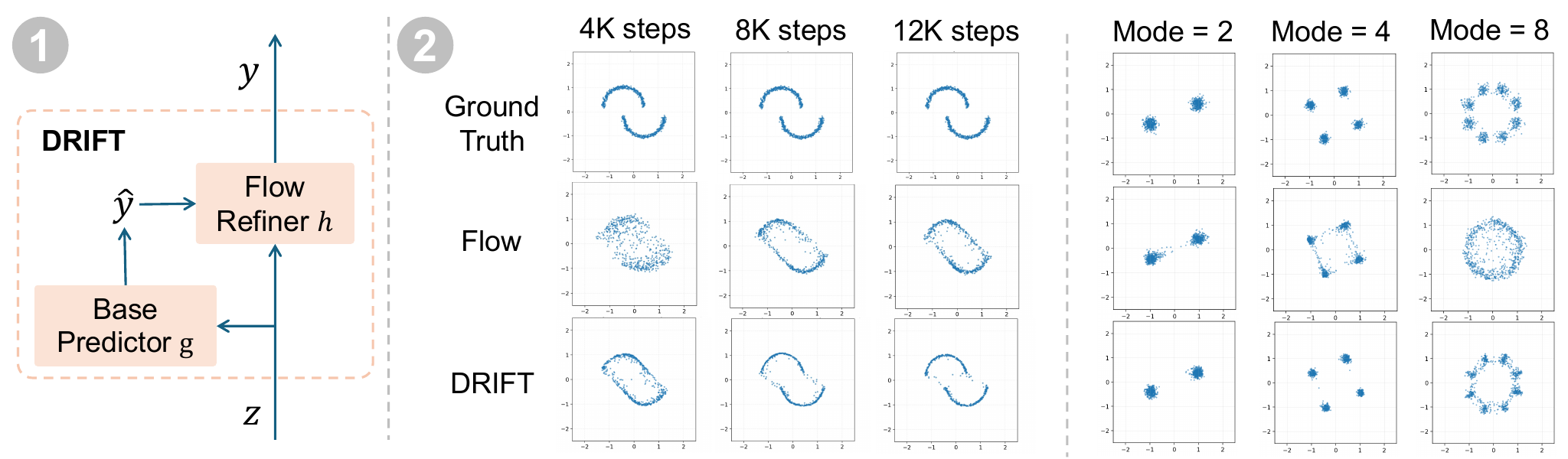}\vspace{-0.5em}
    \caption{\textbf{(1)}: \textbf{Overview of \ours}, which combines a base predictor with a generative refinement module for continuous decoding conditioned on VLM embedding $z$.
    \textbf{(2)}: \textbf{Toy experiments} on a synthetic 2D dataset projected into a high-dimensional space with a fixed, random, column-orthogonal projection matrix, following~\citep{li2025back}. A 3-layer MLP is trained to model the velocity field in the projected space. The projection matrix is unknown to the model and used only for visualization.
    \textbf{(2-Left)}: \ours achieves faster convergence than flow matching. 
    \textbf{(2-Right)}: As the underlying distribution becomes more multimodal, \ours more accurately recovers the target distribution.}
    \label{fig:method_and_toy}\vspace{-1em}
\end{figure}

\textbf{Problem Formulation}.
Our goal is to decode continuous output $y \in \mathbb{R}^c$ from VLM embedding $z$ without modifying existing VLM architecture. 
Our key intuition is that a simple base predictor $g$, built on the learned VLM representation, can provide a strong initial estimation and such estimation can be further refined by a generative model. 

To this end, we propose learning a velocity field $v_\theta(y(t), t, z)$ that transports a Gaussian distribution $\mathcal{N}(\hat{y} \triangleq g(z), \sigma \bm I)$ centered at the coarse prediction $\hat{y} \triangleq g(z)$, towards the target conditional distribution $p(y|z)$. 
Here $\sigma \in \mathbb{R}^c$ scales the variance of the Gaussian and is learned from data. In contrast to standard flow matching, which transports $\mathcal{N}(\bm 0, \bm I)$ to $p(y|z)$, our approach, illustrated in Figure~\ref{fig:method_and_toy}-1, initializes starting distribution using the coarse prediction $\hat{y}$, offering a strong prior over $y$.

\medskip
\textbf{Parameterization}. 
We parameterize the velocity field as
\begin{equation}
    v_\theta(y(t), t, z) = \frac{y_\theta(y(t), t, z) - y(t)}{1-t}, 
    \label{eq:fm:v}
\end{equation}
where $y_\theta(\cdot)$ aims to directly predict the target output $y(t=1)$. This parameterization was recently shown to be competitive with velocity prediction in generation tasks~\citep{li2025back}.

Predicting the target $y$ allows us to add \textit{a skip connection} to the base predictor $g$. This is given by
\begin{equation}
    y_\theta(y(t), t, z) = \alpha \ g(z) + (1-\alpha) \ h_\theta(y(t), t, z).\label{eq:skip}
\end{equation}
$h_\theta(\cdot)$ is time-dependent and can be considered as the \textit{refinement module}. It iteratively refines the residual between $y$ and the coarse prediction $\hat{y}$. $\alpha$ is a learnable gating vector that selectively combines the outputs of the base predictor $g$ and the refinement module $h_\theta(\cdot)$.

\medskip
\textbf{Learning Target}. We consider the following linear interpolant, also referred to as the \textit{bridge state}, 
\begin{equation}
    y(t) = (1-t)( \sigma\epsilon + g(z) ) + ty,
    \label{eq:interpolant}
\end{equation}
where $\epsilon \sim \mathcal{N}(\bm 0, \bm I)$. The target velocity is given by $v^* = d y(t) / dt = y-\sigma\epsilon-\hat y$. We follow the standard flow matching loss in Eq.~\ref{eq:fm:loss}. We sample $t \in \mathcal{U}(0,1)$, construct $y(t)$ by mixing a noise $\epsilon$ with a target output $y$, compute its velocity at $t$ with the corresponding condition $z$ following Eq. \ref{eq:fm:v}, and then minimize the mean squared error (MSE) between the predicted velocity and the target flow $v^*$. The learnable parameters are: (1) $\theta$ in $h_\theta(\cdot)$, which models the residual error; (2) the gating vector $\alpha$, which balances the coarse prediction and its refinement; and (3) the scaling vector $\sigma$, which controls the variance of the initial Gaussian distribution.

It is worth noting that an alternative linear interpolant is $y(t) = (1-t)\sigma\epsilon + t(y-g(z))$. When the base predictor $g(\cdot)$ is fixed, these two interpolants are mathematically equivalent. However, the equivalence no longer holds when $g(\cdot)$ is jointly learned. In that case, they present two conceptually different ways of modeling. Eq.~\ref{eq:interpolant} first ``shifts'' the Gaussian distribution to $\sigma\epsilon+g(z)$ and then models the transport from the shifted distribution to the target distribution. In contrast, the alternative interpolant here directly models the gaps between the initial Gaussian and the distribution of the residual error. Practically, Eq.~\ref{eq:interpolant} offers greater flexibility in modeling. For example, it allows the skip connection in our refinement module (Eq.~\ref{eq:skip}). Empirically, we validate that Eq.~\ref{eq:interpolant} provides minor performance benefits in our ablation studies (see Section~\ref{sec:exp:ablation_study}).

\medskip
\textbf{Joint Training with the Base Predictor}. When the base predictor $g$ is already provided as part of the VLM, we either keep its parameters frozen or update them, 
based on the setting. 
When such a predictor is unavailable, we instantiate $g$ as an MLP regressor, supervise it with an additional MSE loss $\mathcal{L}_{p}$ (with the overall loss as $\Lfm + \mathcal{L}_{p}$), and jointly optimize its parameters. A key challenge in this joint optimization setting is that the model can collapse to relying primarily on the flow matching model while ignoring the coarse predictor $g$. We find that stable learning requires applying gradient stopping to $g(\cdot)$ when computing the linear interpolant using Eq.~\ref{eq:interpolant}.

\medskip
\textbf{Inference}. At inference, we sample Gaussian noise $\epsilon \sim \mathcal{N}(\bm 0, \bm I)$ and initialize $y(0)=g(z)+\sigma \epsilon$. We then solve the ODE $dy(t)/dt = v_\theta(y(t), t, z)$ to obtain the final prediction $y(1)$. 
 
\subsection{Toy Experiments and Theoretical Justification}
We present results of \ours and a vanilla flow matching model on several toy examples in Figure~\ref{fig:method_and_toy}-2. \ours more faithfully captures the underlying data distribution by disentangling its modes.

We provide a detailed theoretical analysis in Appendix~\ref{sec:supp:theory}, assuming a fixed base predictor $g(z)$. Residual refinement strictly reduces prediction error compared with using the base predictor $g(z)$ alone, as long as the bridge state $y(t)$ carries information about the residual $r = y - \hat y$ beyond what $g(z)$ already captures. Intuitively, $y(t)$ is a noisy interpolation between the initial prediction $\hat y$ and the true target $y$, so it inherently contains a partial target signal that the refinement module $h(y(t), t, z)$ can exploit to better approximate $r$. Compared with applying flow matching directly to $y$, when the base predictor already captures the condition-dependent mean structure, both the residual signal $r$ and the residual velocity field have smaller second moments, and the bridge is confined to a narrower tube around the prior. The refinement module therefore only needs to learn a local correction rather than transport noise to the full target, which explains \ours's faster convergence and higher accuracy.

\section{Experiments and Results}
\label{sec:experiments}

\subsection{Results of Vision-Language-Action Tasks}
\label{sec:exp:vla}

\begin{table*}[t!]  
\centering
\vspace{-0.5em}
\caption{\textbf{Results on VLA benchmarks.} We report action success rates, where * indicates reproduced results. \ours-MLP and \ours-Tokenizer use MLP and tokenization-based base predictors, respectively. Across Libero and Simpler benchmarks, \ours consistently improves upon OpenVLA and achieves state-of-the-art performance with the Qwen3-VL-2B backbone.}
\label{tab:vla_main}
  
\resizebox{\textwidth}{!}{
\begin{tabular}{l|c|c|ccccc|c}  
\toprule
\multirow{2}{*}{Method}  & \multirow{2}{*}{Base Model} & \multirow{2}{*}{Action Decoder} & \multicolumn{5}{c|}{Libero~\citep{liu2023libero}$\uparrow$} & \multirow{2}{*}{Simpler~\citep{li24simpler}$\uparrow$}  \\
                         &                &                       & Spatial    & Object & Goal     & Long & AVG   \\ 
    \midrule
    \multicolumn{6}{l}{\it \small\textbf{Methods with VLA pretraining}} \\
    \rowcolor{colorB}
     OpenVLA*~\citep{kim2024openvla}      & Prismatic-7B~\citep{karamcheti2024prismatic}  &  Tokenizer &  83.1 & 88.5  & 76.4 & 54.1 & 75.5 & 7.7  \\
     $\pi_0$~\citep{black2025pi0}         & PaliGemma-3B~\citep{beyer2024paligemma} & Flow &96.8 & 98.8 & 95.8 & 85.2 & 94.1 & 27.1 \\
     $\pi_{0.5}$~\citep{intelligence2025pi} & PaliGemma-3B~\citep{beyer2024paligemma} & Flow & 98.8 & 98.2 & \textbf{98.0} & 92.4 & 96.9 & 57.1 \\
     
    \multicolumn{6}{l}{\it \small\textbf{Methods based on Adaptation}} \\
    VLA-Adaptor~\citep{wang2025vla}  & Qwen-2.5-0.5B~\citep{yang2024qwen2}  &  Diffusion & 97.8 & 99.2 & 97.2 & 95.0 & 97.3 &  N/A  \\
    \rowcolor{colorA}
    VLM4VLA~\citep{zhang2026vlm4vla}  & Qwen3-VL-2B~\citep{bai2025qwen3}    & MLP & N/A & N/A & N/A & 55.8 & N/A & 49.0 \\
    
     VLM4VLA~\citep{zhang2026vlm4vla}  & InternVL3.5-4B~\citep{wang2025internvl3} & MLP & N/A & N/A & N/A & 62.8 & N/A & 57.3 \\
    InternVLA~\citep{chen2025internvla}  & Qwen2.5-VL-3B~\citep{bai2025qwen25vltechnicalreport}  &  Diffusion & 98.0	& 99.0	& 93.8  & 92.6 & 95.8 & 54.2 \\
     
    \midrule
    \rowcolor{colorB}
    \ours                 & OpenVLA-7B~\citep{kim2024openvla} & \ours-Tokenizer & 87.0 & 89.1 & 77.5 & 57.0 & 77.7 & 13.5 \\
    \rowcolor{colorA}
    \ours                 & Qwen3-VL-2B~\citep{bai2025qwen3}   & \ours-MLP & \textbf{98.8} & \textbf{99.2} & 97.0 & \textbf{96.4} & \textbf{97.9} & \textbf{61.5}  \\
    
     \bottomrule
\end{tabular}
}
\vspace{-1.2em}
\end{table*}

VLA tasks aim to predict robot actions conditioned on visual inputs and language instructions. Given an observation of the environment and a task description, a VLA model generates continuous control signals for robotic manipulation. The predicted actions are executed by the robot, after which new observations are collected and fed back into the model, resulting in a closed-loop control process. 

\textbf{Experiment Design}. We instantiate \ours individually with Qwen3-VL-2B-Instruct~\citep{bai2025qwen3} and OpenVLA~\citep{kim2024openvla} as VLM backbones. 
Qwen3-VL is a recent MLLM pretrained on large-scale multimodal corpora, but it does not provide a dedicated base predictor for VLA tasks. We use this model to study how large-scale pretraining benefits continuous action decoding and to evaluate the joint optimization of the base predictor and flow-matching-based refinement.
OpenVLA, in contrast, is a VLA model that predicts actions through tokenization. We include this model to demonstrate how \ours can refine predictions on top of an existing base predictor.

\textbf{Experiment Protocol}. We evaluate \ours on Libero-Spatial, Libero-Object, Libero-Goal, Libero-Long~\citep{liu2023libero}, and Simpler~\citep{li24simpler} with the WidowX robot arm. Both Libero and Simpler are simulation-based benchmarks, with a focus on tabletop object manipulation tasks. 
We train \ours on the LIBERO training set for LIBERO experiments, and on the Bridge and Fractal splits of OXE~\citep{o2024open} when evaluating the policy on Simpler with WidowX.
We report the action success rate as our evaluation metric.

\textbf{Implementation Details.} 
For the experiments on Libero, we train on the Spatial, Object, Goal, and Long suites, respectively, while we train \ours on the Bridge and Fractal splits of Open X-Embodiment (OXE) for the experiments on Simpler WidowX. 
We attach LoRA~\citep{hu2022lora} adapters to the backbone VLM, including the vision encoder and the LLM, and jointly train \ours with the base predictor. 
Appendix~\ref{appx:vla_implementation} provides full implementation details and hyperparameters.

\textbf{Results}. Table~\ref{tab:vla_main} presents the results. 
Built on OpenVLA, \ours achieves consistent improvement across different sub-tasks on Libero and Simpler. Particularly, \ours raises the Libero average from 75.5\% to 77.7\% and Simpler from 7.7\% to 13.5\%.  With Qwen3-VL-2B, \ours reaches 97.9\% average success rate on Libero and 61.5\% on Simpler, achieving new state-of-the-art performance on both benchmarks. We outperform the VLM4VLA with Qwen3-VL-2B by a clear margin, demonstrating the effectiveness of \ours. 
Yet, the major gap between \ours initialized with OpenVLA and Qwen3-VL reveals that the pretraining of the VLM has a strong effect on the downstream task. 
The general improvement in the vision understanding of Qwen3-VL provides strong action representations helpful for \ours, and thus yields better results.

\subsection{Results of Temporal Video Grounding}
\label{sec:exp:tvg}
TVG aims to localize the temporal segment in a video that corresponds to a natural language query. Given a video and a text description, the model predicts continuous start and end timestamps indicating when the described event occurs.

\begin{table*}[t!]
\centering
\caption{\textbf{Results on temporal video grounding.} We evaluate the zero-shot generalization capability of \ours on the Charades-STA~\citep{gao2017charades} and ActivityNet Captions~\citep{krishna2017dense} test splits, and compare against existing baselines. 
Baseline results are taken from ED-VTG~\citep{pramanick2025edvtg}, except for ET-Chat, whose results are reproduced using the released checkpoint under the same evaluation protocol. 
``\# Train'' denotes the size of the temporal grounding adaptation corpus. \ours outperforms ED-VTG, the strongest generalist baseline, on both datasets, with the largest gains observed on R@1@0.3 and mIoU metrics.}
\label{tab:tvg_zs}

\newlength{\tvgActivityMetricWidth}
\settowidth{\tvgActivityMetricWidth}{ActivityNet-Captions~\citep{krishna2017dense}$\uparrow$}
\setlength{\tvgActivityMetricWidth}{\dimexpr\tvgActivityMetricWidth/4-\tabcolsep*3/2\relax}

\resizebox{\textwidth}{!}{
\begin{tabular}{l|c|cccc|*{4}{>{\centering\arraybackslash}p{\tvgActivityMetricWidth}}}
\toprule
\multirow{2}{*}{Method} & \multirow{2}{*}{\# Train} & \multicolumn{4}{c|}{Charades-STA~\citep{gao2017charades}$\uparrow$} & \multicolumn{4}{c}{ActivityNet-Captions~\citep{krishna2017dense}$\uparrow$} \\
                        &                            & R@0.3 & R@0.5 & R@0.7 & mIoU & R@0.3 & R@0.5 & R@0.7 & mIoU \\
\midrule
\multicolumn{10}{l}{\it \small\textbf{Non-generalist (Specialist) Models}} \\
UniVTG~\citep{lin2023univtg}                       & 4.2M  & 44.1 & 25.2 & 10.0 & 27.1 & --   & --   & --   & --   \\
SeViLA~\citep{yu2023sevila}                        & 129M  & --   & --   & --   & --   & 31.6 & 19.0 & 10.1 & 23.0 \\
PSVL~\citep{nam2021psvl}                           & --    & 46.2 & 31.3 & 14.2 & 31.2 & 44.7 & 30.1 & 14.7 & 29.6 \\
LT-ZVG~\citep{ltzvg}                       & --    & 52.9 & 37.2 & 19.3 & 36.0 & 47.6 & 32.6 & \underline{15.4} & 31.8 \\
\midrule
\multicolumn{10}{l}{\it \small\textbf{Generalist (LLM-based) Models}} \\
Video-LLaMA~\citep{zhang2023videollama}            & 2.7M  & 25.2 & 10.6 & 3.4  & 16.8 & 21.9 & 10.8 & 4.9  & 16.5 \\
Video-ChatGPT~\citep{maaz2024videochatgpt}         & 100K  & 27.2 & 6.2  & 1.9  & 19.7 & 19.5 & 10.6 & 4.8  & 14.2 \\
Valley~\citep{luo2023valley}                       & 100K  & 28.4 & 1.8  & 0.3  & 21.4 & 30.6 & 13.7 & 8.1  & 21.9 \\
VideoChat2~\citep{li2024mvbench}                   & 2M    & 38.0 & 14.3 & 3.8  & 24.6 & 40.8 & 27.8 & 9.3  & 27.9 \\
Momentor~\citep{qian2024momentor}                  & 10M   & 42.6 & 26.6 & 11.6 & 28.5 & 42.9 & 23.0 & 12.4 & 29.3 \\
VTimeLLM~\citep{huang2024vtimellm}                 & 170K  & 51.0 & 27.5 & 11.4 & 31.2 & 44.0 & 27.8 & 14.3 & 30.4 \\
TimeChat~\citep{ren2024timechat}                   & 125K  & --   & 32.2 & 13.4 & --   & --   & --   & --   & --   \\
HawkEye~\citep{wang2024hawkeye}                    & 715K  & 50.6 & 31.4 & 14.5 & 33.7 & 49.1 & 29.3 & 10.7 & 32.7 \\
ChatVTG~\citep{qu2024chatvtg}                      & 100K  & 52.7 & 33.0 & 15.9 & 34.9 & 40.7 & 22.5 & 9.4  & 27.2 \\
ED-VTG~\citep{pramanick2025edvtg}                  & 136K  & 59.5 & 39.3 & 19.8 & 40.2 & \underline{52.1} & \underline{33.1} & \textbf{16.0} & \underline{35.2} \\
\rowcolor{colorA} ET-Chat~\citep{liu2024etbench}                      & 164K  & \underline{65.7} & \textbf{45.9} & \underline{20.0} & \underline{42.3} & 24.1 & 12.7 & 6.2  & 18.9 \\
\midrule
\rowcolor{colorA} \ours (ET-Chat)                                   & 100K  & \textbf{67.2} & \underline{44.8} & \textbf{20.6} & \textbf{43.8} & \textbf{57.9} & \textbf{35.8} & 14.1 & \textbf{37.3} \\
\bottomrule
\end{tabular}
}\vspace{-0.5em}

\end{table*}

\textbf{Experiment Design}. We adopt ET-Chat~\citep{liu2024etbench} as the VLM backbone, following the ED-VTG~\citep{pramanick2025edvtg} protocol for fair comparison with prior generalist baselines. ET-Chat natively localizes events by cosine-matching a generated $\langle\text{vid}\rangle$ token's hidden state to frame embeddings. We replace this matching head with an MLP that regresses $(t_\text{start}, t_\text{end})$ from the same hidden state, providing the base prediction that \ours refines. The MLP, flow refiner, and LoRA adapters are jointly trained on TimeLens-100K~\citep{zhang2025timelens}. Appendix~\ref{appx:timelens100k} provides more implementation details.

\textbf{Experiment Protocol}. We evaluate on the Charades-STA~\citep{gao2017charades} and ActivityNet-Captions~\citep{krishna2017dense} test splits, and report Recall@1 at IoU thresholds $\tau \in \{0.3, 0.5, 0.7\}$ (R1@$\tau$) and mean IoU (mIoU). Charades-STA and ActivityNet-Captions training splits are unseen in  TimeLens-100K, so all results are cross-dataset zero-shot.

\textbf{Results.} Table~\ref{tab:tvg_zs} reports the VTG results. \ours (ET-Chat) outperforms ED-VTG, the strongest ED-VTG-reported generalist baseline, on 7 of 8 IoU/mIoU columns, while the only exception is ActivityNet-Captions R@0.7, where \ours trails ED-VTG by 1.9\%. On Charades-STA, where vanilla ET-Chat already exceeds ED-VTG by 6.2\% at R@0.3, \ours adds $+1.5$\% R@0.3 ($65.7\% \to 67.2\%$) and $+1.5$\% mIoU ($42.3\% \to 43.8$\%), with R@0.5 within $\pm 1$\% of ET-Chat. On ActivityNet-Captions, vanilla ET-Chat is weaker (R@0.3 24.1\%, mIoU 18.9\%, and \ours lifts these to 57.9\% and 37.3\%, respectively, surpassing ED-VTG on R@0.3, R@0.5, and mIoU.
These results demonstrate the effectiveness of \ours.

\subsection{Further Demonstration}
\label{sec:exp:further_analysis}

\begin{table*}[t]
\centering

\begin{minipage}[t]{0.48\textwidth}
    \vspace{0pt}
    \centering
    \caption{\textbf{Results on spatial grounding.} Integrating \ours into Qwen3-VL-2B improves performance on RefCOCO benchmarks.}
    \label{tab:spatial_grounding}

    \resizebox{\linewidth}{!}{
    \begin{tabular}{l|cccc}  
        \toprule
        Method & refcoco & refcoco+ & refcocog & AVG \\
        \midrule
        Qwen-VL~\citep{bai2025qwen3}
            & 89.6 & 81.8 & 85.5 & 85.6 \\
        Qwen-VL (w/ \ours)
            & \textbf{91.7} & \textbf{85.8} & \textbf{87.9} & \textbf{88.5} \\
        \bottomrule
    \end{tabular}
    }
\end{minipage}
\hfill
\begin{minipage}[t]{0.48\textwidth}
    \vspace{0pt}
    \centering
    \caption{\textbf{Results with world action models.} Integrating \ours with FastWAM improves performance on the Libero benchmarks.}
    \label{tab:fast_wan}

    \resizebox{\linewidth}{!}{
    \begin{tabular}{l|ccccc}  
        \toprule
        Method & Spatial & Object & Goal & Long & AVG \\
        \midrule
        FastWAM~\citep{yuan2026fast}
            & 97.6 & 99.4 & 95.6 & 95.0 & 96.9 \\
        FastWAM (w/ \ours)
            & \textbf{99.2} & 99.0 & \textbf{98.2} & \textbf{96.2} & \textbf{98.2} \\
        \bottomrule
    \end{tabular}
    }
\end{minipage}

\vspace{-1em}
\end{table*}

\begin{figure}[t!]
    \centering
    \includegraphics[width=0.95\linewidth]{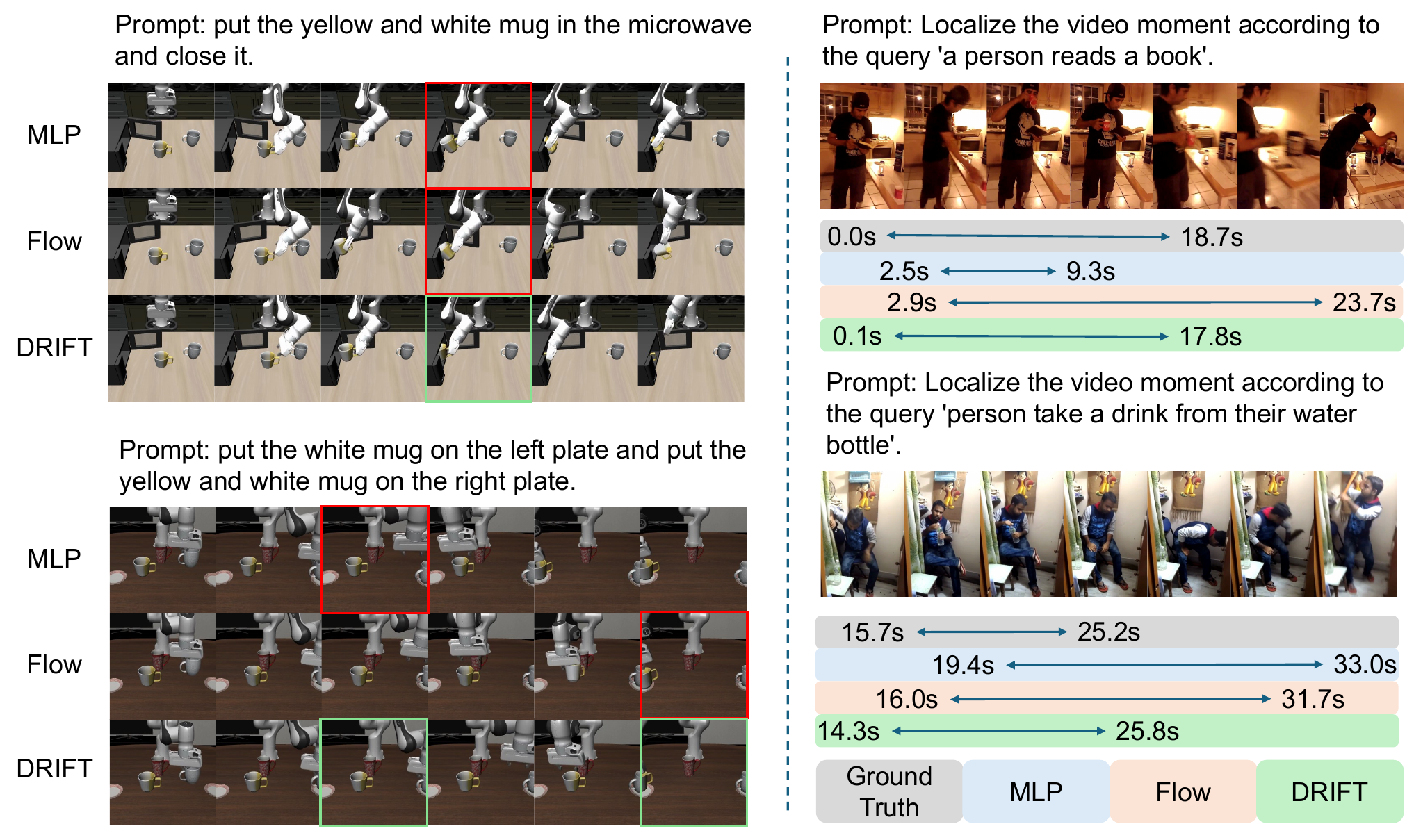}\vspace{-0.8em}
    \caption{\textbf{Qualitative results.} \textbf{Left:} Integrating a base predictor and a flow-based refiner enables \ours to generate more accurate and stable action trajectories, improving action precision and overall task success rates. \textbf{Right:} Through iterative refinement of the coarse prediction, \ours produces sharper and more accurate event temporal boundaries.}
    \label{fig:vla_visualization}
\end{figure}

\textbf{2D Spatial Grounding}. We further adapt \ours to 2D spatial grounding tasks. Specifically, we attach \ours to the Qwen3-VL-2B model. We then train and evaluate the model on the RefCOCO series datasets~\citep{kazemzadeh2014referitgame, mao2016generation, yu2016modeling}. As shown in Table~\ref{tab:spatial_grounding}, \ours consistently improves the performance of Qwen3-VL by a clear margin across all benchmarks, demonstrating the versatility and generality of our framework. More experimental details can be found in Appendix~\ref{appx:spatial_grounding_impl}.

\textbf{World Action Model}. We further explore adapting \ours to recent world action models. Specifically, we integrate \ours with FastWAM~\citep{yuan2026fast}. Unlike the other models considered in this paper, FastWAM already employs a flow-matching-based decoder for predicting continuous actions. In this setting, we introduce an additional base predictor using an MLP and jointly optimize it together with FastWAM’s original flow decoder. 
We train the model following FastWAM's training recipe, and present the results in Table~\ref{tab:fast_wan}. \ours consistently improves FastWAM on the Libero benchmark suite, further demonstrating the generality of our framework. Appendix~\ref{appx:fastwam_implementation_details} provides more experimental details.

\textbf{Qualitative Results}. 
Figure~\ref{fig:vla_visualization}-left shows actions from different decoders on Libero-Long. Small action errors can cause task failures: a poor grasp angle blocks placing the cup into the microwave (top), and a slight release offset misses the plate (bottom). By refining the base prediction, \ours yields more accurate actions and higher success rates. Similarly, Figure~\ref{fig:vla_visualization}-right compares predictions from different decoders on TVG tasks. Exact temporal boundaries are often ambiguous. Through iterative refinement, \ours produces sharper temporal boundaries, leading to improved localization accuracy and overall performance.

\textbf{Inference Latency}. Table~\ref{tab:inference_latency} compares end-to-end latency. On VLA, \ours improves success by 8.1 percentage points over MLP while increasing latency from 28.4 to 31.4\,ms (10.5\%). Relative to direct flow, it improves success by 2.7 points for 0.3\,ms additional latency. On TVG, \ours improves mIoU by 1.5 points while increasing latency from 603.3 to 605.8\,ms (0.4\%). Both evaluations use 10 forward-Euler steps. End-to-end latency on NVIDIA H200 GPUs is averaged over 1{,}000 VLA passes across 10 tasks and 510 TVG queries from 200 videos.

\subsection{Ablation Study}
\label{sec:exp:ablation_study}
\begin{table*}[t]
\centering
\begin{minipage}[t]{0.48\textwidth}
\vspace{0pt}
\caption{\textbf{Inference latency and decoder structures.} VLA: Libero-Long success rate (\%, mean $\pm$ standard deviation over three runs) with different action decoders based on Qwen3-VL-2B. TVG: Charades-STA mIoU (\%) with ET-Chat-3.8B. }
\label{tab:inference_latency}
\centering
\small
\setlength{\tabcolsep}{3pt}
\renewcommand{\arraystretch}{1.0}
\sbox0{%
\begin{tabular}{@{}llcc@{}}
\toprule
Task & Decoder & Perf. (\%) & ms \\
\midrule
\multirow{4}{*}{VLA} & MLP & 88.20 $\pm$ 0.00 & 28.4 \\
& Diffusion & 92.23 $\pm$ 0.64 & 33.5 \\
& Flow & 93.67 $\pm$ 0.35 & 31.1 \\
& \ours & \textbf{96.33 $\pm$ 0.23} & 31.4 \\
\midrule
\multirow{2}{*}{TVG} & ET-Chat & 42.3 & 603.3 \\
& \ours (ET-Chat) & \textbf{43.8} & 605.8 \\
\bottomrule
\end{tabular}
}
\ifdim\wd0>\linewidth
    \resizebox{\linewidth}{!}{\usebox0}%
\else
    \usebox0%
\fi
\end{minipage}
\hfill
\begin{minipage}[t]{0.48\textwidth}
\vspace{0pt}
\caption{\textbf{\ours structural design.} Libero-Long success rate (\%) with Qwen3-VL-2B as base model. We compare Direct Flow and \ours, velocity ($v$) and signal ($x$) prediction, skip connections, and the base predictor loss $\mathcal{L}_p$. The best configuration uses \ours with $x$-prediction, a skip connection, and $\mathcal{L}_p$.}
\label{tab:drift_design}
\centering
\small
\setlength{\tabcolsep}{3pt}
\renewcommand{\arraystretch}{1.0}
\sbox0{%
\begin{tabular}{ll|cc|cc}
\toprule
& & \multicolumn{2}{c|}{Direct Flow} & \multicolumn{2}{c}{\ours} \\
\cmidrule{3-6}
& & v-pred & x-pred & v-pred & x-pred \\
\midrule
\multirow{2}{*}{w/ skip}
& w/ $\mathcal{L}_{p}$  & N/A & N/A & N/A & \textbf{93.8} \\
& w/o $\mathcal{L}_{p}$ & N/A & N/A & N/A & \textbf{93.6} \\
\midrule
\multirow{2}{*}{w/o skip}
& w/ $\mathcal{L}_{p}$  & 92.4 & 91.2 & 91.0 & 92.2 \\
& w/o $\mathcal{L}_{p}$ & 91.0 & 91.2 & 93.2 & 93.4 \\
\bottomrule
\end{tabular}
}
\ifdim\wd0>\linewidth
    \resizebox{\linewidth}{!}{\usebox0}%
\else
    \usebox0%
\fi
\end{minipage}
\end{table*}

We ablate \ours on Libero-Long with Qwen3-VL-2B, using the training setup of \cref{sec:exp:vla}. 

\textbf{Decoder Structure and Prediction Variance.} 
We compare \ours against MLP regression, Flow Matching~\citep{lipman2023flow}, and Diffusion~\citep{peebles2023scalable} decoders. The VLA rows of Table~\ref{tab:inference_latency} summarize the results, averaged across three runs. Both Flow Matching and Diffusion provide strong generative baselines for continuous decoding. Nevertheless, \ours consistently outperforms both.

\begin{table*}[t]
\centering
\begin{minipage}[t]{0.48\textwidth}
\vspace{0pt}
\caption{\textbf{Co-training effect.} Libero-Long success (\%) with Qwen3-VL-2B. Joint training improves the MLP base predictor over a standalone MLP, while the full \ours decoder further improves performance.}
\label{tab:co_training}
\end{minipage}
\hfill
\begin{minipage}[t]{0.48\textwidth}
\vspace{0pt}
\caption{\textbf{Prior quality.} Libero-Long success (\%) with Qwen3-VL-2B. The random base predictor is untrained and has 0\% success on its own. A trained prior improves \ours over both the random-prior variant and direct flow.}
\label{tab:prior_quality}
\end{minipage}
\par
\begin{minipage}[t]{0.48\textwidth}
\vspace{0pt}
\centering
\small
\setlength{\tabcolsep}{3pt}
\renewcommand{\arraystretch}{1.0}
\sbox0{%
\begin{tabular}{lc}
\toprule
Decoder & Libero-Long  (\%) \\
\midrule
MLP & 88.2 \\
MLP as base & 93.0 \\
\ours & \textbf{96.4} \\
\bottomrule
\end{tabular}
}
\ifdim\wd0>\linewidth
    \resizebox{\linewidth}{!}{\usebox0}%
\else
    \usebox0%
\fi
\end{minipage}
\hfill
\begin{minipage}[t]{0.48\textwidth}
\vspace{0pt}
\centering
\small
\setlength{\tabcolsep}{3pt}
\renewcommand{\arraystretch}{1.0}
\sbox0{%
\begin{tabular}{@{}lr@{}}
\toprule
Decoder & Libero-Long  (\%) \\
\midrule
Flow & 94.2 \\
\ours (random base predictor) & 94.0 \\
\ours & \textbf{96.4} \\
\bottomrule
\end{tabular}
}
\ifdim\wd0>\linewidth
    \resizebox{\linewidth}{!}{\usebox0}%
\else
    \usebox0%
\fi
\end{minipage}
\end{table*}

\textbf{\ours structural design}. We compare design choices of \ours with a shorter 30K-step training schedule. We vary (1) the linear interpolant:  $y_t = (1-t)\epsilon + t(y - \hat{y})$ (Direct Flow) vs.\ our design (\ours) $y_t = (1-t)(\epsilon + \hat{y}) + ty$ in Eq. \ref{eq:interpolant}, (2) velocity prediction ($v$-pred) vs.\ signal prediction ($x$-pred), (3) skip connection vs.\ no skip connection, and (4) with or without the base predictor loss $\mathcal{L}_p$. When without $\mathcal{L}_p$ loss, we allow the gradients from refinement module to back-propagate to the base predictor.
Table~\ref{tab:drift_design} presents the design comparison. Among the valid configurations, the best result is our current design, $x$-prediction, skip-connection variant with $L_p$, which reaches 93.8\%. Without the skip connection, the same interpolant with $x$-pred reaches 93.4\%, still outperforming the Direct Flow alternatives at 91.2\%. The effect of $\mathcal{L}_p$ is modest: with the skip connection, it changes 93.6\% to 93.8\%, while in the no-skip setting, the best score is obtained without it.

\textbf{\ours Learns Improved Action Representations}. Table~\ref{tab:co_training} shows the effects of jointly training the base predictor with the flow matching based refinement module. A standalone MLP baseline achieves 88.2\% on Libero-Long, while the same MLP, when co-trained as the base predictor within \ours, improves to 93.0\%. Incorporating the full \ours decoder further boosts performance to 96.4\%. We hypothesize that joint training encourages the VLM to learn more informative action representations, thereby improving the performance of the base predictor itself.

\textbf{Effect of Prior Quality.} We test the behavior of \ours when the base predictor provides no useful prior. 
We construct a deliberately weak predictor. Specifically, we train \ours without applying the auxiliary loss $\mathcal{L}_p$ to the base predictor (MLP), while also stopping gradients from the refinement module to the base predictor. The base predictor thus remains untrained throughout optimization and achieves 0\% accuracy on Libero-Long. \ours then reaches 94.0\% (Table~\ref{tab:prior_quality}), close to direct flow matching (94.2\%), indicating that \ours effectively degenerates to a flow matching model and achieving comparable performance.

\section{Conclusion}
\label{sec:conclusion}

In this paper, we presented \ours, a general framework for adapting pretrained VLMs to continuous decoding tasks. \ours introduces a lightweight residual flow adapter that augments a base predictor with a flow-matching-based refinement module. It progressively refines an initial coarse prediction by modeling residual errors around a strong prior. This residual formulation simplifies generative modeling, improves optimization and learning efficiency, and enables effective continuous decoding across diverse tasks and architectures. Through extensive experiments on visual grounding and robotic control tasks spanning MLLMs, VLAs, and WAMs, we demonstrate consistent improvements over strong regression- and generative-based baselines. We hope our work sheds light on post-training strategies for extending VLMs beyond discrete token generation toward precise continuous outputs.

\bibliographystyle{iclr2027_conference}
\bibliography{main}

\newpage
\appendix

\section*{Appendix: Supplementary Materials}

\subsection*{Overview}
This appendix provides the statistical analysis, implementation details, and additional qualitative results supporting \ours. The materials are organized as follows:
\begin{itemize}[leftmargin=*, itemsep=0.5em]
    \item \textbf{Appendix~\ref{sec:supp:theory}: A Statistical Perspective on Residual Refinement.}
    Notation and assumptions, population-level and finite-sample comparisons with direct MLP prediction, and comparisons with direct flow matching covering Bayes equivalence, target second moments, and residual-bridge localization.
    \item \textbf{Appendix~\ref{appx:impl}: Implementation Details.}
    \ours architecture and task-specific data, model, training, inference, and evaluation settings for vision-language-action models, temporal video grounding, spatial grounding, and FastWAM.
    \item \textbf{Appendix~\ref{appx:visualization}: Additional Visualization.}
    Qualitative temporal video grounding results illustrating boundary refinement and shared error cases involving ambiguous queries or inaccurate annotations.
\end{itemize}

\section{A Statistical Perspective on Residual Refinement}
\label{sec:supp:theory}

\subsection{Notation and Setup}
\label{sec:supp:notation}

Let $z\in\cZ$ denote the conditional representation produced by the MLLM encoder and $y\in\R^c$ the continuous target.
A base predictor $g(z)$ outputs the point estimate $\hat{y}\triangleq g(z)$, and the residual is $r\triangleq y - g(z)$.
Unless otherwise stated, the theoretical discussion assumes a continuous-valued base predictor $g(z)\in\R^c$. The token-based base predictor is treated as a practical variant evaluated empirically.

In \ours, we sample noise $\epsilon\sim\mathcal{N}(\bm 0,\bm I)$ and construct the prior-shifted bridge
\begin{equation}
  y_t = (1-t)\bigl(g(z)+\sigma\epsilon\bigr) + ty, \qquad t\in[0,1],
  \label{eq:supp:bridge}
\end{equation}
where $\sigma\in\R_+^c$ are noise scales.
Let $X\triangleq(y_t,t,z)$ denote the input to the refinement network~$h$.
The refinement network learns the mapping $h:X\mapsto r$, and the final prediction is $\bar{y}=g(z)+h(X)$.

We define the \emph{unweighted risk} and \emph{weighted risk} of any measurable predictor~$h$ as
\begin{equation}
\small
  \Risk(h) \triangleq \E\bigl[\|h(X)-r\|_2^2\bigr]
  = \E\bigl[\|g(z)+h(X)-y\|_2^2\bigr]
  = \E\bigl[\|\bar y-y\|_2^2\bigr],
  \quad
  \cL(h) \triangleq \E\!\left[\frac{\|h(X)-r\|_2^2}{(1-t)^2}\right].
  \label{eq:supp:risk}
\end{equation}
The weighted form $\cL(h)$ arises from the equivalence between velocity matching and signal prediction in flow matching (cf.\ \cref{eq:fm:v,eq:fm:loss} in the main text).
We write $w(t)\triangleq 1/(1-t)^2$.

For direct flow matching (direct FM), the bridge is $\tilde{y}_t=(1-t)\sigma\epsilon + ty$, the input is $\tilde{X}\triangleq(\tilde{y}_t,t,z)$, and a predictor $H:\tilde{X}\mapsto y$ is trained with weighted risk $\cL_{\mathrm{FM}}(H)\triangleq\E\bigl[w(t)\|H(\tilde{X})-y\|_2^2\bigr]$.

The \emph{direct MLP} baseline corresponds to setting $h\equiv 0$, so that the final output is simply~$g(z)$.

\subsection{Assumptions}
\label{sec:supp:assumptions}

\begin{definition}[Assumptions~A]
\label{def:assumptions}
The following conditions are assumed throughout.

\begin{enumerate}
\item[\textbf{A1}] \textbf{Truncated time sampling.}
There exists $\tau\in(0,1)$ such that $t\in[0,1-\tau]$, which ensures $1\le w(t)\le\tau^{-2}$.

\item[\textbf{A2}] \textbf{Bounded targets and outputs.}
There exist constants $B_r, B_y, B_h, B_H>0$ such that almost surely
\[
  \|r\|_2\le B_r, \qquad \|y\|_2\le B_y, \qquad \|h(X)\|_2\le B_h, \qquad \|H(\tilde{X})\|_2\le B_H.
\]

\item[\textbf{A3}] \textbf{Finite complexity.}
The refinement function class~$\cH$ and the direct FM class~$\cF$ have finite empirical Rademacher complexities $\Rad_n(\cH)$ and $\Rad_n(\cF)$, respectively.
\end{enumerate}
\end{definition}

\noindent
\textbf{A1} avoids the weight divergence as $t\to 1$; in practice this is handled by truncated time sampling or numerical clipping.
\textbf{A2} states that training targets and network outputs are bounded, which can equivalently be stated as sub-Gaussian tail bounds.
\textbf{A3} ensures that empirical risk minimization yields nontrivial generalization bounds.

\subsection{Advantage over Direct MLP}
\label{sec:supp:vs-mlp}

We first introduce the \emph{Bayes residual predictor}
\begin{equation}
  m(X) \triangleq \E[r \mid X],
  \label{eq:supp:bayes-residual}
\end{equation}
which is the mean-squared-error-optimal prediction of the residual given the bridge state.
If $m(X)\not\equiv 0$, then the bridge state carries information about the residual beyond what $g(z)$ already captures, and the refinement network can exploit this to reduce risk.

\begin{lemma}[Orthogonal decomposition of squared loss]
\label{lem:ortho}
Let $m(X)=\E[r\mid X]$. Then for any measurable~$h$,
\begin{align}
  \Risk(h) &= \Risk(m) + \E\bigl[\|h(X)-m(X)\|_2^2\bigr], \label{eq:supp:ortho-unweighted}\\
  \cL(h) &= \cL(m) + \E\bigl[w(t)\|h(X)-m(X)\|_2^2\bigr]. \label{eq:supp:ortho-weighted}
\end{align}
\end{lemma}

\begin{proof}
We prove the unweighted case; the weighted case follows identically since $w(t)$ is $X$-measurable.
Write
\[
  h(X)-r = \bigl(h(X)-m(X)\bigr) + \bigl(m(X)-r\bigr).
\]
Expanding the squared norm:
\[
  \|h(X)-r\|_2^2 = \|h(X)-m(X)\|_2^2 + \|m(X)-r\|_2^2 + 2\langle h(X)-m(X),\, m(X)-r\rangle.
\]
Taking expectations, it suffices to show the cross term vanishes.
Since $h(X)-m(X)$ is $X$-measurable,
\[
  \E\bigl[\langle h(X)-m(X),\, m(X)-r\rangle \mid X\bigr]
  = \bigl\langle h(X)-m(X),\, \E[m(X)-r\mid X]\bigr\rangle = 0,
\]
because $\E[m(X)-r\mid X] = m(X) - \E[r\mid X] = 0$ by definition.
By the tower property, the unconditional cross term is also zero, giving
$\Risk(h)=\Risk(m)+\E\|h(X)-m(X)\|_2^2$.
\end{proof}

\begin{theorem}[Population-level improvement over direct MLP]
\label{thm:pop-improvement}
Direct MLP corresponds to the zero residual predictor $h\equiv 0$. Then
\begin{align}
  \Risk(0) - \inf_h \Risk(h) &= \E\bigl[\|m(X)\|_2^2\bigr], \label{eq:supp:pop-gap-unweighted}\\
  \cL(0) - \inf_h \cL(h) &= \E\bigl[w(t)\|m(X)\|_2^2\bigr]. \label{eq:supp:pop-gap-weighted}
\end{align}
In particular, whenever $m(X)\not\equiv 0$, we have $\inf_h\Risk(h)<\Risk(0)$, i.e., \ours strictly improves over direct MLP at the population level.
\end{theorem}

\begin{proof}
By \cref{lem:ortho} with $h\equiv 0$:
\[
  \Risk(0) = \Risk(m) + \E\bigl[\|m(X)\|_2^2\bigr].
\]
Since $\inf_h \Risk(h) = \Risk(m)$ (the Bayes predictor attains the infimum), we obtain \cref{eq:supp:pop-gap-unweighted}.
The weighted case is analogous: $\cL(0)=\cL(m)+\E[w(t)\|m(X)\|_2^2]$ and $\inf_h\cL(h)=\cL(m)$.
\end{proof}

\noindent
\textbf{Interpretation.}
The improvement gap $\E[\|m(X)\|_2^2]$ measures how much residual structure the bridge state can explain.
Direct MLP stops at the coarse prediction $g(z)$; \ours further leverages the bridge state to refine the residual.
As long as the bridge state carries any predictive information about the residual, \ours achieves strictly lower risk.

\begin{theorem}[Finite-sample sufficient condition for improvement over direct MLP]
\label{thm:finite-sample-mlp}
Under Assumptions~A, let $S=\{(X_i,r_i,t_i)\}_{i=1}^n$ be an i.i.d.\ training sample.
Define the empirical risk minimizer $\hat{h}\in\argmin_{h\in\cH}\hat{\cL}(h)$ where $\hat{\cL}(h)\triangleq\frac{1}{n}\sum_{i=1}^n w(t_i)\|h(X_i)-r_i\|_2^2$,
and the best-in-class predictor $h^{\star}_{\cH}\in\argmin_{h\in\cH}\cL(h)$ with approximation error $\App_{\cH}\triangleq\cL(h^{\star}_{\cH})-\cL(m)\ge 0$.

Then there exists a constant $C>0$ such that, with probability at least $1-\delta$,
\begin{equation}
  \cL(\hat{h}) - \cL(h^{\star}_{\cH}) \le C\tau^{-2}\left[(B_r+B_h)\Rad_n(\cH) + (B_r+B_h)^2\sqrt{\frac{\log(2/\delta)}{n}}\right].
  \label{eq:supp:finite-excess}
\end{equation}
Furthermore, since $w(t)\ge 1$,
\begin{equation}
  \Risk(0) - \Risk(\hat{h}) \ge \E\bigl[\|m(X)\|_2^2\bigr] - \App_{\cH} - C\tau^{-2}\left[(B_r+B_h)\Rad_n(\cH) + (B_r+B_h)^2\sqrt{\frac{\log(2/\delta)}{n}}\right].
  \label{eq:supp:finite-gap}
\end{equation}
Therefore, whenever
\begin{equation}
  \E\bigl[\|m(X)\|_2^2\bigr] > \App_{\cH} + C\tau^{-2}\left[(B_r+B_h)\Rad_n(\cH) + (B_r+B_h)^2\sqrt{\frac{\log(2/\delta)}{n}}\right],
  \label{eq:supp:finite-condition}
\end{equation}
we have $\Risk(\hat{h})<\Risk(0)$, i.e., \ours with finite-sample ERM strictly outperforms direct MLP.
\end{theorem}

\begin{proof}
The proof proceeds in three steps.

\medskip\noindent\textbf{Step 1: Uniform deviation bound.}
Define the per-sample loss $\ell_h(X,r,t)=w(t)\|h(X)-r\|_2^2$.
By \textbf{A1} and \textbf{A2}, $1\le w(t)\le \tau^{-2}$ and $\|h(X)-r\|_2\le B_h+B_r$, so $0\le\ell_h\le\tau^{-2}(B_h+B_r)^2$.

For fixed $(r,t)$, the map $u\mapsto\phi_{r,t}(u)=w(t)\|u-r\|_2^2$ is Lipschitz in $u$ with constant at most $2\tau^{-2}(B_h+B_r)$.
By the vector-valued Rademacher contraction inequality,
\[
  \Rad_n\bigl(\{\ell_h: h\in\cH\}\bigr) \le C_1\tau^{-2}(B_h+B_r)\Rad_n(\cH),
\]
where $C_1$ is an absolute constant.
The standard Rademacher generalization bound then gives, with probability at least $1-\delta$,
\begin{equation}
  \sup_{h\in\cH}\bigl|\cL(h)-\hat{\cL}(h)\bigr| \le C_2\tau^{-2}\left[(B_r+B_h)\Rad_n(\cH)+(B_r+B_h)^2\sqrt{\frac{\log(2/\delta)}{n}}\right].
  \label{eq:supp:uniform-bound}
\end{equation}

\medskip\noindent\textbf{Step 2: Excess risk of ERM.}
Since $\hat{h}$ minimizes empirical risk, $\hat{\cL}(\hat{h})\le\hat{\cL}(h^{\star}_{\cH})$.
Applying \cref{eq:supp:uniform-bound} to both $\hat{h}$ and $h^{\star}_{\cH}$:
\[
  \cL(\hat{h}) \le \hat{\cL}(\hat{h})+\Delta_n \le \hat{\cL}(h^{\star}_{\cH})+\Delta_n \le \cL(h^{\star}_{\cH})+2\Delta_n,
\]
where $\Delta_n$ denotes the right-hand side of \cref{eq:supp:uniform-bound}.
Absorbing the factor of 2 into the constant~$C$ yields \cref{eq:supp:finite-excess}.

\medskip\noindent\textbf{Step 3: Converting to improvement over direct MLP.}
By \cref{lem:ortho}, $\Risk(\hat{h})-\Risk(m)=\E\|\hat{h}(X)-m(X)\|_2^2$.
Since $w(t)\ge 1$,
\[
  \E\bigl[\|\hat{h}(X)-m(X)\|_2^2\bigr] \le \E\bigl[w(t)\|\hat{h}(X)-m(X)\|_2^2\bigr] = \cL(\hat{h})-\cL(m).
\]
Decomposing the right-hand side:
\[
  \cL(\hat{h})-\cL(m) = \bigl(\cL(\hat{h})-\cL(h^{\star}_{\cH})\bigr) + \bigl(\cL(h^{\star}_{\cH})-\cL(m)\bigr) = \bigl(\cL(\hat{h})-\cL(h^{\star}_{\cH})\bigr) + \App_{\cH}.
\]
Hence $\Risk(\hat{h})-\Risk(m) \le \App_{\cH} + (\text{RHS of \cref{eq:supp:finite-excess}})$.

From \cref{thm:pop-improvement}, $\Risk(0)-\Risk(m)=\E[\|m(X)\|_2^2]$.
Therefore
\[
  \Risk(0)-\Risk(\hat{h}) = \bigl(\Risk(0)-\Risk(m)\bigr) - \bigl(\Risk(\hat{h})-\Risk(m)\bigr) \ge \E\bigl[\|m(X)\|_2^2\bigr] - \App_{\cH} - (\text{RHS of \cref{eq:supp:finite-excess}}),
\]
which gives \cref{eq:supp:finite-gap}. When \cref{eq:supp:finite-condition} holds, the right-hand side is strictly positive.
\end{proof}


\subsection{Comparison with Direct Flow Matching}
\label{sec:supp:vs-fm}

We next compare \ours with direct full-target flow matching.
The purpose of this subsection is not to claim that residualization changes the unrestricted population optimum.
In fact, the two objectives are Bayes-equivalent after a deterministic shift of the bridge state.
The advantage suggested by the teaser visualization is instead geometric and statistical: after the base predictor removes the predictable component of the target, the flow refiner learns a smaller residual correction and its bridge remains localized around the coarse anchor.
This distinction is important when the target $y$ has already been normalized to a bounded range such as $[-1,1]^c$; worst-case range bounds alone do not explain the advantage, but residual second moments and bridge radii still can.

\begin{proposition}[Bayes equivalence of \ours and direct FM]
\label{prop:bayes-equiv}
With unrestricted measurable function classes,
\begin{equation}
  \inf_h \cL(h) = \inf_H \cL_{\mathrm{FM}}(H).
  \label{eq:supp:bayes-equiv}
\end{equation}
\end{proposition}

\begin{proof}
By definition, the \ours bridge and the direct FM bridge satisfy
\[
  y_t = \tilde{y}_t + (1-t)g(z),
\]
so given $(t,z)$, the states $y_t$ and $\tilde{y}_t$ differ by a deterministic shift.

For any residual predictor~$h$, define the full-target predictor
$H_h(\tilde{X})\triangleq h(X)+g(z)$, where $X=(\tilde{y}_t+(1-t)g(z),t,z)$.
Then
\[
  H_h(\tilde{X})-y = h(X)+g(z)-y = h(X)-r,
\]
so $\|H_h(\tilde{X})-y\|_2^2 = \|h(X)-r\|_2^2$.
Multiplying by $w(t)$ and taking expectations gives $\cL_{\mathrm{FM}}(H_h)=\cL(h)$.

Conversely, for any full-target predictor~$H$, define
$h_H(X)\triangleq H(\tilde{X})-g(z)$, where $\tilde{X}=(y_t-(1-t)g(z),t,z)$.
Then
\[
  h_H(X)-r = H(\tilde{X})-g(z)-(y-g(z))=H(\tilde{X})-y,
\]
so $\cL(h_H)=\cL_{\mathrm{FM}}(H)$.
This loss-preserving bijection implies \cref{eq:supp:bayes-equiv}.
\end{proof}

\noindent
\textbf{Interpretation.}
Since the unrestricted population optima coincide, \ours should not be presented as having a lower Bayes risk than direct FM.
The relevant comparison is instead about the target that the learned generative component has to model.
Direct FM learns the full target transport, whereas \ours assigns the predictable component to the base predictor and asks the flow network to learn the remaining residual transport.

\begin{theorem}[Lower-variance residual transport]
\label{thm:tighter-bound}
Let
\[
  \mu(z) \triangleq \E[y\mid z],
  \qquad
  r \triangleq y-g(z),
\]
and assume that $\epsilon$ is independent of $(y,z)$ with $\E[\epsilon]=0$.
Define the direct-FM target velocity and the residual-FM target velocity as
\[
  v^{\star}_{\mathrm{dir}} \triangleq y-\sigma\epsilon,
  \qquad
  v^{\star}_{\mathrm{res}} \triangleq y-g(z)-\sigma\epsilon = r-\sigma\epsilon.
\]
Then the signal targets satisfy
\begin{align}
  \E\|y\|_2^2
  &= \E\!\left[\mathrm{Tr}\,\operatorname{Cov}(y\mid z)\right]
     + \E\|\mu(z)\|_2^2, \label{eq:supp:direct-signal-second-moment}\\
  \E\|r\|_2^2
  &= \E\!\left[\mathrm{Tr}\,\operatorname{Cov}(y\mid z)\right]
     + \E\|\mu(z)-g(z)\|_2^2. \label{eq:supp:residual-signal-second-moment}
\end{align}
The corresponding velocity targets satisfy
\begin{align}
  \E\|v^{\star}_{\mathrm{dir}}\|_2^2
  &= \E\|y\|_2^2 + \E\|\sigma\epsilon\|_2^2, \label{eq:supp:direct-velocity-second-moment}\\
  \E\|v^{\star}_{\mathrm{res}}\|_2^2
  &= \E\|r\|_2^2 + \E\|\sigma\epsilon\|_2^2. \label{eq:supp:residual-velocity-second-moment}
\end{align}
Consequently,
\begin{equation}
  \E\|v^{\star}_{\mathrm{dir}}\|_2^2
  -
  \E\|v^{\star}_{\mathrm{res}}\|_2^2
  =
  \E\|y\|_2^2 - \E\|r\|_2^2
  =
  \E\|\mu(z)\|_2^2 - \E\|\mu(z)-g(z)\|_2^2.
  \label{eq:supp:velocity-gap}
\end{equation}
Thus, whenever the base predictor explains more conditional-mean energy than the zero predictor,
\begin{equation}
  \E\|\mu(z)-g(z)\|_2^2 < \E\|\mu(z)\|_2^2,
  \label{eq:supp:mean-contraction-condition}
\end{equation}
the residual correction target and the residual velocity target have smaller second moment than their direct-FM counterparts.
Moreover,
\begin{equation}
  \operatorname{Cov}(r\mid z)=\operatorname{Cov}(y\mid z),
  \label{eq:supp:conditional-cov-unchanged}
\end{equation}
so residualization does not remove the irreducible conditional uncertainty of the target distribution; it removes only the condition-dependent mean structure that the base predictor already captures.
\end{theorem}

\begin{proof}
We first prove the two signal-target identities.
By the conditional second-moment decomposition,
\[
  \E\bigl[\|y\|_2^2\mid z\bigr]
  = \mathrm{Tr}\,\operatorname{Cov}(y\mid z)+\|\E[y\mid z]\|_2^2
  = \mathrm{Tr}\,\operatorname{Cov}(y\mid z)+\|\mu(z)\|_2^2.
\]
Taking expectation over $z$ gives \cref{eq:supp:direct-signal-second-moment}.
For the residual target, since $g(z)$ is deterministic given $z$,
\[
  \E[r\mid z]=\E[y-g(z)\mid z]=\mu(z)-g(z),
  \qquad
  \operatorname{Cov}(r\mid z)=\operatorname{Cov}(y\mid z).
\]
Applying the same conditional second-moment decomposition to $r$ gives
\[
  \E\bigl[\|r\|_2^2\mid z\bigr]
  = \mathrm{Tr}\,\operatorname{Cov}(y\mid z)+\|\mu(z)-g(z)\|_2^2,
\]
and taking expectation over $z$ gives \cref{eq:supp:residual-signal-second-moment}.

For the velocity targets, expand the squared norms:
\[
  \E\|y-\sigma\epsilon\|_2^2
  = \E\|y\|_2^2 + \E\|\sigma\epsilon\|_2^2
    -2\E\langle y,\sigma\epsilon\rangle,
\]
and
\[
  \E\|r-\sigma\epsilon\|_2^2
  = \E\|r\|_2^2 + \E\|\sigma\epsilon\|_2^2
    -2\E\langle r,\sigma\epsilon\rangle.
\]
Because $\epsilon$ is independent of $(y,z)$ and has zero mean, both cross terms vanish.
This proves \cref{eq:supp:direct-velocity-second-moment,eq:supp:residual-velocity-second-moment}.
Subtracting the residual identity from the direct identity yields \cref{eq:supp:velocity-gap}.
The last claim follows from the fact that subtracting a deterministic function of $z$ does not change the conditional covariance given $z$.
\end{proof}

\noindent
\textbf{Normalized-output implication.}
The comparison above does not require $\|r\|_2$ to be uniformly smaller than $\|y\|_2$.
Even when $y$ is normalized to $[-1,1]^c$, the direct decoder still has to learn the condition-dependent mean transport $\mu(z)$, while the residual decoder only has to learn $\mu(z)-g(z)$.
If $g(z)$ is a useful base predictor, then $\E\|\mu(z)-g(z)\|_2^2$ can be much smaller than $\E\|\mu(z)\|_2^2$ despite both targets living in bounded ranges.

\begin{theorem}[Localization of the residual bridge]
\label{thm:localized-bridge}
For a fixed $t\in[0,1]$, define the residual-bridge radius around the coarse anchor and the direct-FM bridge radius around the noise origin as
\[
  \rho_{\mathrm{res}}(t) \triangleq \E\|y_t-g(z)\|_2^2,
  \qquad
  \rho_{\mathrm{dir}}(t) \triangleq \E\|\tilde{y}_t\|_2^2.
\]
Then
\begin{align}
  \rho_{\mathrm{res}}(t)
  &= (1-t)^2\E\|\sigma\epsilon\|_2^2 + t^2\E\|r\|_2^2, \label{eq:supp:residual-radius}\\
  \rho_{\mathrm{dir}}(t)
  &= (1-t)^2\E\|\sigma\epsilon\|_2^2 + t^2\E\|y\|_2^2. \label{eq:supp:direct-radius}
\end{align}
Therefore, if $\E\|r\|_2^2\le \eta\E\|y\|_2^2$ for some $\eta<1$, then
\begin{equation}
  \rho_{\mathrm{res}}(t)
  \le
  \rho_{\mathrm{dir}}(t)-t^2(1-\eta)\E\|y\|_2^2.
  \label{eq:supp:radius-gap}
\end{equation}
Moreover, for any radius $a>0$,
\begin{equation}
  \mathbb{P}\!\bigl(\|y_t-g(z)\|_2>a\bigr)
  \le
  \frac{(1-t)^2\E\|\sigma\epsilon\|_2^2+t^2\E\|r\|_2^2}{a^2}.
  \label{eq:supp:tube-bound}
\end{equation}
\end{theorem}

\begin{proof}
Using the \ours bridge,
\[
  y_t-g(z)
  = (1-t)\sigma\epsilon + t(y-g(z))
  = (1-t)\sigma\epsilon + tr.
\]
Expanding the squared norm and using the independence and zero mean of $\epsilon$ gives
\[
  \E\|y_t-g(z)\|_2^2
  = (1-t)^2\E\|\sigma\epsilon\|_2^2+t^2\E\|r\|_2^2,
\]
which proves \cref{eq:supp:residual-radius}.
Similarly, the direct FM bridge satisfies
\[
  \tilde{y}_t=(1-t)\sigma\epsilon+ty,
\]
so the same expansion gives \cref{eq:supp:direct-radius}.
Substituting $\E\|r\|_2^2\le \eta\E\|y\|_2^2$ into \cref{eq:supp:residual-radius} and comparing with \cref{eq:supp:direct-radius} gives \cref{eq:supp:radius-gap}.
Finally, \cref{eq:supp:tube-bound} follows from Markov's inequality applied to the nonnegative random variable $\|y_t-g(z)\|_2^2$.
\end{proof}

\noindent
\textbf{Teaser interpretation.}
\cref{thm:localized-bridge} formalizes the qualitative picture in which residual refinement stays in a smaller tube around the coarse anchor, while direct generative decoding must transport noise toward the full target region.
The statement is intentionally about the \emph{correction} and the \emph{bridge geometry}, not about artificially reducing the final conditional output diversity.
At the Bayes optimum, the final conditional distribution should still match the true target distribution; \ours changes which part of the mapping is learned by the deterministic coarse prediction and which part is left to the generative refiner.

\section{Implementation Details}
\label{appx:impl}
\subsection{\ours Architecture}
Figure~\ref{fig:resflow_structure} illustrates the architecture of \ours. \ours consists of two main components: a base predictor and a flow refiner. The base predictor first produces a coarse continuous estimate $\hat{y}$ from the MLLM representation. The flow refiner then iteratively predicts the remaining residual from $\hat{y}$ to the ground-truth target.

The flow refiner contains three components: embedders, self-attention layers, and an output head. The embedders map the condition representation, the flow time $t$, and the bridge state $y_t$ into a shared hidden space. The self-attention layers aggregate information across these embedded tokens. Finally, the output head maps the updated representation to the flow prediction, which is combined with the $\hat{y}$ through the skip connection to produce the final prediction.

\begin{figure}[t!]
    \centering
    \includegraphics[width=0.7\linewidth]{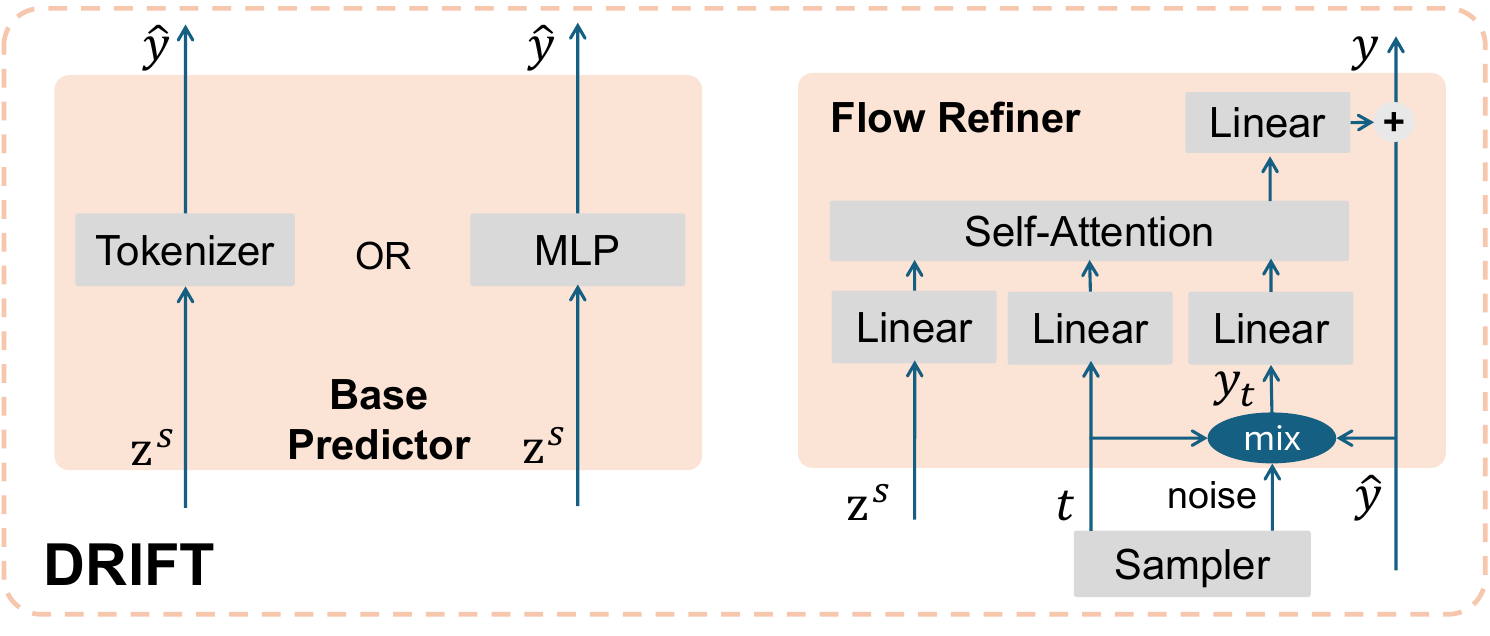}
    \caption{The overview of the \ours's structure.}
    \label{fig:resflow_structure}
\end{figure}

\subsection{VLA Implementation Details}
\label{appx:vla_implementation}
\paragraph{Training Data.}
We follow InternVLA-M1~\citep{chen2025internvla} to prepare the training data.
For the experiments on Libero~\citep{liu2023libero}, we train \ours on the Spatial, Object, Goal, and Long suites, respectively. Spatial, Object, Goal, and Long contain 432, 454, 428, and 379 training trajectories, respectively, and yield 53K, 67K, 52K, and 101K training samples, respectively. We normalize the action trajectories with the max and min values.

For the experiments on Simpler~\citep{li24simpler} WidowX, we follow the setting used in~\citep{chen2025internvla} and train the \ours on the Bridge and Fractal splits of Open X-Embodiment (OXE)~\citep{o2024open}. Bridge and Fractal contain 60K and 87K training trajectories, respectively, and yield 1.28M and 3.4M training samples, respectively. We normalize the action trajectories with the 99th percentile value for the trajectories.

\paragraph{Base VLMs.}
We use Qwen3-VL-2B-Instruct~\citep{bai2025qwen3} and OpenVLA~\citep{kim2024openvla} as the base models. For OpenVLA, when training and evaluating on Simpler, we use the checkpoint pretrained on OXE. When training and evaluating on each Libero suite, we use the corresponding Libero-finetuned OpenVLA checkpoint. During \ours adaptation, we add LoRA~\citep{hu2022lora} adapters to both the vision encoder and the LLM backbone. We use LoRA rank 32, $\alpha$= 64 for Libero experiments and rank 128, $\alpha$= 256 for Simpler experiments.

\paragraph{Special Tokens.} For Qwen3-VL-2B-Instruct, we use the special token \texttt{<think>}, which already exists in the tokenizer but is not used during pre-training. For OpenVLA, we add a special token by extending the tokenizer, and we use both the added special token and the action tokens in the OpenVLA output as special tokens. We extract the hidden state of the special token from the last LLM layer as the condition $z^s$.

\paragraph{Base Predictor.}
When applying \ours to Qwen3-VL, we train a learnable MLP as the base predictor, which directly predicts a continuous action $\hat{y}$. We optimize this base predictor with MSE loss. When applying \ours to OpenVLA, we use OpenVLA's tokenizer as the base predictor. The tokenizer discretizes each continuous value into 256 bins represented by 256 action tokens. During training, we sample the action from the multinomial distribution predicted by the model and use it as $\hat{y}$. We optimize the tokenizer-based base predictor with negative log-likelihood loss.

\paragraph{Flow Refiner.}
The flow refiner uses 4 self-attention layers with 8 heads and a hidden dimension of 768. We optimize the flow refiner with MSE loss between the target velocity $v^*$ and the predicted velocity $v_\theta(y(t), t, z)$.

\paragraph{Training Hyper-parameters.}
We leverage both the main view and the wrist view to train \ours with Qwen3-VL-2B~\citep{bai2025qwen3}, and \ours learns to predict the next 8 and 4 consecutive actions in the future for the Libero and Simpler settings, respectively.
We only use the main view when training \ours with OpenVLA~\citep{kim2024openvla}, and \ours predicts one future action to match the original setting of OpenVLA.

We train \ours with batch size 16 and learning rate $4\times 10^{-5}$ for 30K steps on Spatial, Object, and Goal, and for 60K steps on Long. For Simpler WidowX, we train with batch size 128, learning rate $4\times 10^{-5}$, and 50K steps. We use cosine learning-rate decay with 10\% warmup steps. All the experiments are trained with AdamW~\citep{loshchilovdecoupled} optimizer on 8*H100 in 8-16 hours.

\paragraph{Inference and Evaluation.}
We use $K=10$ forward-Euler steps and one stochastic noise sample per query. We evaluate \ours on Libero-Spatial, Libero-Object, Libero-Goal, Libero-Long, and Simpler WidowX. On Libero, we run 50 trials for each of the 10 subtasks in every suite. On Simpler WidowX, we run 24 trials for each of 4 subtasks. We report the success rate as the metric.

\paragraph{Baselines.} 
We compare against two groups of baselines: (1) models with large-scale robotic pretraining, including OpenVLA~\citep{kim2024openvla}, $\pi_0$~\citep{black2025pi0}, and $\pi_{0.5}$~\citep{intelligence2025pi} and (2) methods based on adaptation such as VLA-Adaptor~\citep{wang2025vla}, VLM4VLA~\citep{zhang2026vlm4vla}, and InternVLA~\citep{chen2025internvla}. The results for the adaptation-based methods are taken from the corresponding paper. The results for OpenVLA are re-evaluated by us using the corresponding checkpoints. We collect the results for $\pi_0$ from previous work~\citep{chen2025internvla}. We gather the results for $\pi_{0.5}$ from the official openpi repository~\citep{intelligence2025pi} and previous work~\citep{xu2025stare}.

\subsection{TVG Implementation Details}
\label{appx:vtg_impl}

We provide the implementation details for \ours on Temporal Video Grounding (TVG). We evaluate \ours under the ED-VTG protocol with ET-Chat as the MLLM backbone.

\paragraph{Training Data.}
\label{appx:timelens100k}
We train the TVG variant of \ours on TimeLens-100K~\citep{zhang2025timelens}. TimeLens-100K contains approximately $20$K source videos and $100$K query--segment temporal grounding annotations. Its source videos are collected from a mixture of video-language datasets, including CosMo-Cap~\citep{wang2024cosmo}, InternVid-VTime~\citep{wang2023internvid}, DiDeMo~\citep{hendricks2017didemo}, QuerYD~\citep{oncescu2021queryd}, and HiREST~\citep{zala2023hirest}. The temporal-grounding annotations are MLLM-generated end-to-end: an MLLM first identifies temporally distributed events in each video and then produces grounding queries with timestamp pairs. We use TimeLens-100K only as the training corpus and evaluate on the ED-VTG benchmarks, Charades-STA~\citep{gao2017charades} and ActivityNet-Captions~\citep{krishna2017dense}. The TimeLens-100K training annotations are distinct from the evaluation annotations used in these benchmarks.

\paragraph{Base VLMs.}
We use ET-Chat~\citep{liu2024etbench} as the backbone. We freeze the ET-Chat visual encoder and frame compressor, apply LoRA to the LLM attention projections, and train the LoRA adapters jointly with the flow refiner. We add three temporal special tokens, \texttt{<CT>}, \texttt{<VALUE>}, and \texttt{</CT>}, and train the LLM with its standard next-token objective to emit these markers. We use LoRA rank $32$, $\alpha=64$, and dropout $0.05$.

\paragraph{Input and Condition Representation.}
We follow the ET-Chat video pipeline: videos are sampled at 1 FPS, encoded with EVA-CLIP ViT-G/14, compressed by ET-Chat's frame compressor, and then fed into Phi-3-Mini-4K-Instruct (3.8B) as the LLM backbone. The frame compressor consists of a Q-Former with learnable queries, a context aggregator, and a projector. For temporal grounding, the target text marks each continuous boundary with \texttt{<CT><VALUE></CT>}. We use the LLM hidden states at the \texttt{<VALUE>} positions as condition $z^s$.

\paragraph{Base Predictor.}
The coarse prediction is produced by a small MLP head over the \texttt{<VALUE>} hidden states. The prior head predicts the temporal segment in center-width coordinates, applies a $1.2\times$ width expansion factor, and converts the result to start-end coordinates.

\paragraph{Flow Refiner.}

The flow refiner uses 2 self-attention layers with GELU activations, $\text{nhead}=16$, and hidden dimension equal to the LLM's $\text{hidden\_size}/2$.
It represents each temporal segment in center-width coordinates and predicts the velocity in residual center-width coordinates. 

\paragraph{Training Hyper-parameters.}
We train \ours on TimeLens-100K with per-device batch size $1$, learning rate $5\times 10^{-4}$ for 4 epochs with AdamW optimization on 8 A100 GPUs. The training run takes approximately 10 hours. The training objective uses flow MSE on the velocity.

\paragraph{Inference and Evaluation.}
The start and end timestamps are normalized to $[-1,1]$ by the video duration during training. At inference time, we unnormalize the prediction back to seconds and clip the final $(t_s,t_e)$ to the valid range $[0,\text{duration}]$. We use $K=10$ forward-Euler steps. Integration is performed in center-width coordinates. The result is converted back to start-end coordinates and clipped to the valid video-time range.

ET-Chat~\citep{liu2024etbench} is instruction-tuned on E.T. Instruct 164K. According to the ET-Chat paper, its TVG samples are drawn from DiDeMo, TACoS, QueryD, and NaQ. ET-Chat does not report ActivityNet-Captions performance in the original paper, and thus we obtain the ET-Chat baseline in Table~\ref{tab:tvg_zs} by running the released checkpoint under the ED-VTG evaluation protocol.

We report Recall@1 at different IoU thresholds and mean IoU. Recall@1 at IoU threshold $\tau$ measures the fraction of queries whose top-1 predicted interval has IoU greater than $\tau$ with the ground-truth interval. mIoU is the mean IoU over all evaluated queries.

\paragraph{Baselines.} We report results for a strong set of generalist zero-shot baselines (Video-LLaMA~\citep{zhang2023videollama}, Video-ChatGPT~\citep{maaz2024videochatgpt}, Valley~\citep{luo2023valley}, VideoChat2~\citep{li2024mvbench}, Momentor~\citep{qian2024momentor}, VTimeLLM~\citep{huang2024vtimellm}, ET-Chat~\citep{liu2024etbench}, TimeChat~\citep{ren2024timechat}, HawkEye~\citep{wang2024hawkeye}, ChatVTG~\citep{qu2024chatvtg}, ED-VTG~\citep{pramanick2025edvtg}) on Charades-STA and ActivityNet Captions, while additionally including specialist models for reference, including UniVTG~\citep{lin2023univtg}, SeViLA~\citep{yu2023sevila}, PSVL~\citep{nam2021psvl}, and LT-ZVG~\citep{ltzvg}. 
The most relevant comparison is the vanilla ET-Chat baseline evaluated under the same protocol, since \ours builds upon ET-Chat as its backbone. Among all baselines, ED-VTG achieves the best performance.
Baseline results are taken from ED-VTG~\citep{pramanick2025edvtg}, except for ET-Chat, whose results are reproduced using the released checkpoint under the same evaluation protocol. 

\subsection{Spatial Grounding Implementation Details}
\label{appx:spatial_grounding_impl}
\paragraph{Training Data.}
We use the training sets of RefCOCO~\citep{kazemzadeh2014referitgame}, RefCOCO+~\citep{mao2016generation}, and RefCOCOg~\citep{yu2016modeling}, which contain 120K, 120K, and 80K training samples, respectively. All the bounding boxes' coordinates are normalized between 0 and 1, and with the bounding box format of [cx,cy,w,h].

\paragraph{Base VLM.}
We use Qwen3-VL-2B-Instruct as the base model. We add the LoRA~\citep{hu2022lora} adaptor with rank 32, $\alpha$= 64, on both the vision encoder and the LLM backbone.

\paragraph{Special Tokens.}
We add one additional special token <think>, which already exists in the tokenizer but was never used during training in the Instruct version of the model. We extract the hidden state of the special token from the last layer of the LLM as condition $z^s$.

\paragraph{Base Predictor.}
We train a learnable MLP as the base predictor, which directly predicts $\hat{y}$. We apply MSE loss to update the base predictor.

\paragraph{Flow Refiner.}
We use 8-layer self-attention, with 8 heads and hidden dimension of 768 in the flow refiner. We optimize the flow refiner with MSE loss between the target velocity $v^*$ and the predicted velocity $v_\theta(y(t), t, z)$. 

\paragraph{Training Hyper-parameters.}
We train \ours with batch size 64 and learning rate $4\times 10^{-5}$ for 1 epoch (about 5K steps). We use cosine-decay learning rate scheduling with 3\% warm-up steps. All the experiments are trained with the AdamW optimizer on 8 A100s in 2 hours.

\paragraph{Inference and evaluation} We use $K=10$ forward-Euler steps. We use one stochastic noise sample per query (deterministic). Following the setting used in Qwen3-VL, we evaluate \ours on the validation set of RefCOCO, RefCOCO+, and RefCOCOg, and report the Accuracy@1 as the metric.

\subsection{FastWAM VLA Implementation Details}
\label{appx:fastwam_implementation_details}
\paragraph{Training Data.}
Following the training setting of the FastWAM~\citep{yuan2026fast}, we use all the splits (Spatial, Object, Goal, and Long) in Libero to train the model. In total, it has 272K training samples. 

\paragraph{Base VLM.}
We use the same architecture for the video generator (VideoDiT) and the action generator (ActionDiT) as FastWAM. We follow FastWAM's implementation to initialize the model before training.

\paragraph{Base Predictor.}
We use a 2-layer QFormer, with 16 learnable queries, 8 heads, and a hidden dimension of 1024, to aggregate the information from the embedded text embeddings and the vision tokens encoded by the VAE~\citep{wan2025wan}. We then apply a mean operation on the query tokens from the QFormer. We send the averaged query token to the MLP to predict $\hat{y}$.

\paragraph{Flow Refiner.}
We use the ActionDiT in the FastWAM structure as the flow refiner, and we optimize the flow refiner with MSE loss between the target velocity $v^*$ and the predicted velocity $v_\theta(y(t), t, z)$.

\paragraph{Training, Inference and Evaluation.}
We use the same training and evaluation setting as FastWAM. We evaluate \ours on Libero-Spatial, Libero-Object, Libero-Goal, and Libero-Long. On Libero, we run 50 trials for each of the 10 subtasks in every suite. We report the success rate as the metric. All the experiments are trained on 8*H200 in 24 hours.

\section{Additional Visualization}
\label{appx:visualization}




Figure~\ref{fig:tvg_additional_visual} presents additional visualizations on TVG, comparing the start and end timestamps predicted by different modules. Samples 1--4 show that \ours achieves more accurate temporal boundaries by iteratively refining the coarse prediction using the flow refiner. Samples 5--6 illustrate error patterns shared by all decoders. These errors are mainly caused by ambiguous or inaccurate annotations: in Sample 5, the action \textit{takes a bite of the sandwich} lasts only about two seconds, but the ground-truth boundary annotates it as a five-second interval; in Sample 6, the query \textit{sits on a bed} is ambiguous, as it can refer either to the action of sitting down or to the state of already being seated. These cases suggest that further improving TVG performance may require both more carefully curated training and validation annotations, as well as models that are better able to handle ambiguous natural-language queries.

\begin{figure}[t!]
    \centering
    \includegraphics[width=1.0\linewidth]{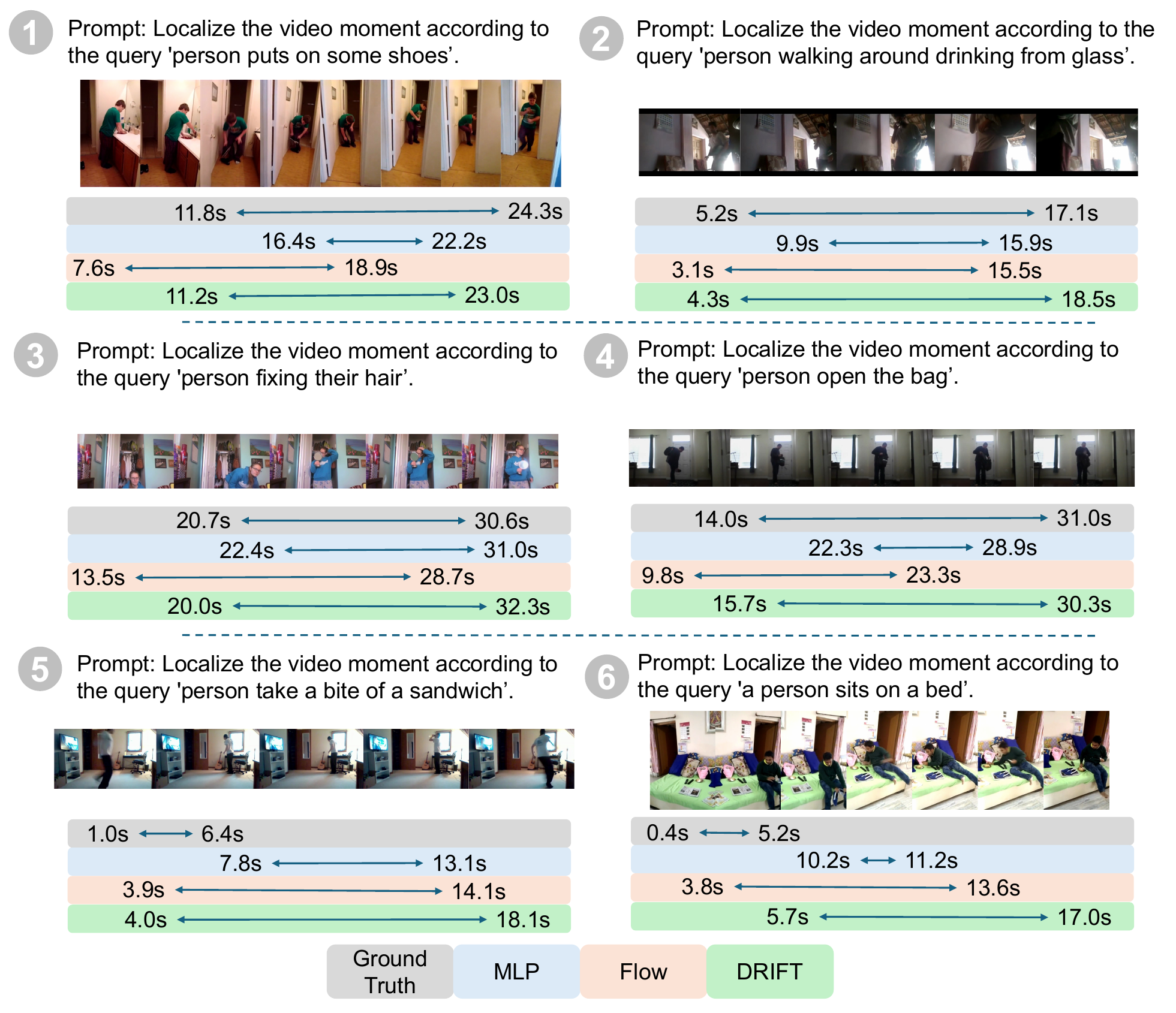}
    \caption{Additional visualization on TVG.}
    \label{fig:tvg_additional_visual}
\end{figure}

\end{document}